\documentclass[12pt]{article}
\usepackage{epsfig, 
amssymb,amsmath, bm}
\usepackage[sort&compress,square,comma,authoryear]{natbib}
\usepackage{ntheorem,delarray}
\usepackage{color}
\usepackage[margin=1.0in]{geometry}

\makeindex\nonstopmode

\long\def\killtext#1{}

\newtheorem{theorem}{Theorem}[section]

\newtheorem{lemma}[theorem]{Lemma}

\newtheorem{claim}[theorem]{Claim}
\newtheorem{corollary}[theorem]{Corollary}

\theorembodyfont{\rmfamily}
\newtheorem{remark}[theorem]{Remark}
\newtheorem{exa}[theorem]{Example}

\newtheorem{defn}[theorem]{Definition}

\newenvironment{example}{\begin{exa}}{\hfill$\blacklozenge$\end{exa}}
\newenvironment{proof}{\noindent{\bf Proof.}}{\hfill$\square$\medskip}

\def\Pr{{\sf Pr}}\def\E{{\sf E}}
\def\eps{\varepsilon}
\def\R{{\mathbb R}}
\def\Z{{\mathbb Z}}
\def\J{\hbox{\rm JOIN}}
\def\L{\hbox{\rm LINK}}

\def\ve{{\varepsilon}}
\def\e{{\sf exp}}
\long\def\killtext#1{}

\begin{document}

\title{Cortical Computation via Iterative Constructions}

\author{Christos Papadimitrou\thanks{UC Berkeley, christos@berkeley.edu, supported in part by NSF award CCF-1408635 and by Templeton Foundation grant 3966.} \and Samantha Petti\thanks{Georgia Tech, \{spetti3,vempala\}@gatech.edu, supported in part by NSF awards CCF-1217793 and EAGER-1415498.} \and Santosh Vempala\footnotemark[2]}

\maketitle

\begin{abstract}  We study Boolean functions of an arbitrary number of input variables that can be realized by simple iterative constructions based on constant-size primitives. This restricted type of construction needs little global coordination or control and thus is a candidate for neurally feasible computation. Valiant's construction of a majority function can be realized in this manner and, as we show, can be generalized to any uniform threshold function. We study the rate of convergence, finding that while linear convergence to the correct function can be achieved for any threshold using a fixed set of primitives, for quadratic convergence, the size of the primitives must grow as the threshold approaches $0$ or $1$. We also study finite realizations of this process and the learnability of the functions realized. We show that the constructions realized are accurate outside a small interval near the target threshold, where the size of the construction grows as the inverse square of the interval width. This phenomenon, that errors are higher closer to thresholds (and thresholds closer to the boundary are harder to represent), is a well-known cognitive finding. 
\end{abstract}

\section{Introduction}


\paragraph{Cortical computation.}
Among the many unexplained abilities of the cortex are learning complex patterns and invariants from relatively few examples. This  is manifested in a range of cognitive functions including visual and auditory categorization, motor learning and language. In spite of the highly varied perceptual and cognitive tasks accomplished, the substrate appears to be relatively uniform in the distribution and type of cells. How could these $80$ billion cells organize themselves so effectively? 
 
Cortical computation must therefore be highly distributed, require little {\em synchrony} (number of pairs of events that must happen in lock-step across neurons), little {\em global control} (longest chain of events that must happen in sequence) and be based on very simple primitives  \citep{PV15b}.  Assuming that external stimuli are parsed as sets of binary sensory features, our central question is the following: 

\smallskip

{\em What functions can be represented and learned by algorithms so simple that one could imagine them happening in the cortex?} 

\smallskip

Perhaps the most natural  primitives are the AND and OR functions on two input variables. These functions are arguably neurally plausible. They were studied as JOIN and LINK by \cite{Valiant94, Valiant00, Valiant05, FeldmanV09}, who showed how to implement them in the neuroidal model. An {\em item} is a collection of neurons (corresponding to a {\em neural assembly} in neuroscience) that represents some learned or sensed concept. Given two items $A,B$, the JOIN operation forms a new item $C=\J(A,B)$, which ``fires" when both $A$ and $B$ fire, i.e., $C$ represents $A \wedge B$. $\L(A,B)$ captures association, and causes $B$ to fire whenever $A$ fires. By setting $\L(A,C)$ and $\L(B,C)$, we achieve that $C$ is effectively $A \vee B$. While the precise implementation and neural correlates of JOIN and LINK are unclear, there is evidence that the brain routinely engages in hierarchical memory formation. 

\paragraph{Monotone Boolean functions.}
Functions constructed by recursive processes based on AND/OR trees have been widely studied in the literature, motivated by the design of reliable circuits as in \citep{Moore1956} and more recently, understanding the complexity-theoretic limitations of monotone Boolean functions.  One line of work studies the set of functions that could be the limits of recursive processes, where at each step, the leaves of  a tree are each replaced by constant-size functions. \cite{Moore1956}, showed that a simple recursive construction leads to a threshold function, which can be applied to construct stable circuits. \cite{Valiant84a} used their $4$-variable primitive function $(A \vee B) \wedge (C \vee D)$ to derive a small depth and size threshold function that evaluates to $1$ if at least $(2-\phi) \approx 0.38$ fraction of the inputs are set to $1$ and to zero otherwise. The depth and size were $O(\log n)$ and $O(n^{5.3})$ respectively. Calling it the {\em amplification method}, \cite{Boppana1985} showed that Valiant's construction is optimal. \cite{Dubiner1992} extended the lower bound to classes of read-once formulae. \cite{Hoory2006} gave smaller size Boolean {\em circuits} (where each gate can have fan out more than $1$), of size $O(n^3)$ for the same threshold function. \cite{Luby1998} gave an alternative analysis of Valiant's construction along with applications to coding. The construction of a Boolean formula was extended by  \cite{Servedio2004} to monotone linear threshold functions, in that they can be approximated on most inputs by monotone Boolean formulae of polynomial size.  \cite{Friedman1986} gave more efficient constructions for threshold functions with small thresholds. 

Savicky gives conditions under which the limit of such a process is the uniform distribution on all Boolean functions with $n$ inputs \citep{Savicky1987, Savicky1990} (see also \cite{Brodsky2005, Fournier2009}).
In a different application,  
\cite{Goldman1993} showed how to use properties of these constructions to identify read-once formulae from their input-output behavior. 

\paragraph{Our work.} 
Unlike previous work, where a single constant-sized function is chosen and applied recursively, we will allow constructions that randomly choose one of two constant-sized functions. To be neurally plausible, our constructions are {\em bottom-up} rather than {\em top-down}, i.e., at each step, we apply a constant-size function to an existing set of outputs. In addition, the algorithm itself must be very simple --- our goal is not to find ways to realize all Boolean functions or to optimize the size of such realizations. Here we address the following questions: 
What functions of $n$ input items can be constructed in this iterative manner?
Can arbitrary uniform threshold functions be realized? What size and depth of iterative constructions suffices to guarantee accurate computations? Can such functions and constructions be learned from examples, where the learning algorithm is also neurally plausible?

Our rationale for uniform threshold functions is two-fold. First, uniform threshold functions are fundamental in computer science and likely also for cognition. Second, the restriction to JOIN and LINK as primitives ensures that any resulting function will be monotone since negation is not possible in this framework. Moreover, if we require the construction to be symmetric, it would seem that the only obtainable family of Boolean functions are uniform thresholds. However, as we will see, there is a surprise here, and in fact we can get {\em staircase functions}, i.e., functions that take value $p_i$ on the interval $(a_i, a_{i+1})$  where $a_0=0<a_1<a_2<\dots <a_k< a_{k+1}=1$ and $0=p_0<p_1< p_2< \dots< p_{k-1}<p_k=1$.

To be able to describe our results precisely, we begin with a definition of iterative constructions.

\subsection{Iterative constructions}

A sequence of AND/OR operations can be represented as a tree.  Such a tree $T$ with $n$ leaves naturally computes a function $g_T: \{0, 1\}^n \to \{0,1\}$. 
We can build larger trees in a neurally plausible way by using a set of small AND/OR trees as building blocks. Let $C$ be a probability distribution on a finite set of trees. We define an {\em iterative tree for $C$} as follows.

\killtext{
\begin{figure}[h]
\centering
\includegraphics[scale=.75]{joinlinkandor.eps}
\caption{The item $\J( \L(A, B), \L (C,D))$ computes $(A \vee B ) \wedge (C \vee D)$.}\label{vtree}
\end{figure}
}



\begin{figure}[h!]
\fbox{\parbox{\textwidth}{
{\bf IterativeTree($L,m,C,X$):}\\[0.1in]
For each level $j$ from $1$ to $L$, apply the following iteration $m$ times:\\ 
(level $0$ consists of the input items $X$)
\begin{enumerate}
\item Choose a tree $T$ according to $C$. 
\item Choose items at random from the items on level $j-1$. 
\item Build the tree $T$ with these items as leaves. 
\end{enumerate}
}}
\end{figure}


The construction of small AND/OR trees 
is a decentralized process requiring a short sequence of steps, i.e., the synchrony and control parameters are small. Therefore, we consider them to be neurally plausible. 

The iterative tree construction has a well-defined sequence of {\em levels}, with items from the next level having leaves only in the current level. A construction that needs even less coordination is the following: the probability that an item participates in future item creation decays exponentially with time. The weight of an item starts at $1$ when it is created and decays by a factor of $e^{-\alpha}$ each time unit. We refer to such constructions as {\em exponential} iterative constructions. An extreme version of this, which we call {\em wild} iterative construction, is to have $\alpha=0$, i.e, all items are equally likely to participate in the creation of new items. Figure \ref{wild and exp} illustrates these constructions. 
\begin{figure}[h!]
\fbox{\parbox{\textwidth}{
{\bf ExponentialConstruction($k, C, \alpha$):}\\

Initialize the weights of input items to $1$.\\
Construct $k$ items as follows:
\begin{enumerate}
\item Choose a tree $T$ according to $C$. 
\item  Choose the leafs for $T$ independently from existing items with probability proportional to the weight of the item.
\item Build the tree $T$ with these items as leaves.
\item Multiply the weight of every item by $e^{-\alpha}$. 
 
\end{enumerate}
}}
\end{figure}

\begin{figure}[h!]
\centering
\includegraphics[scale=.75]{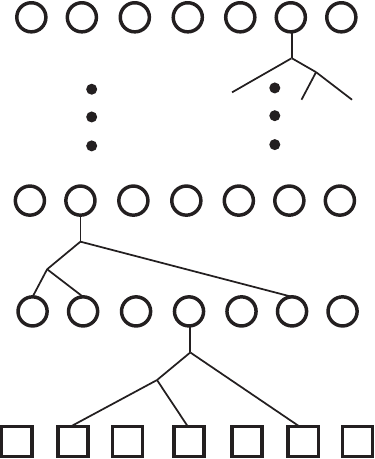}
\quad
\includegraphics[scale=.35]{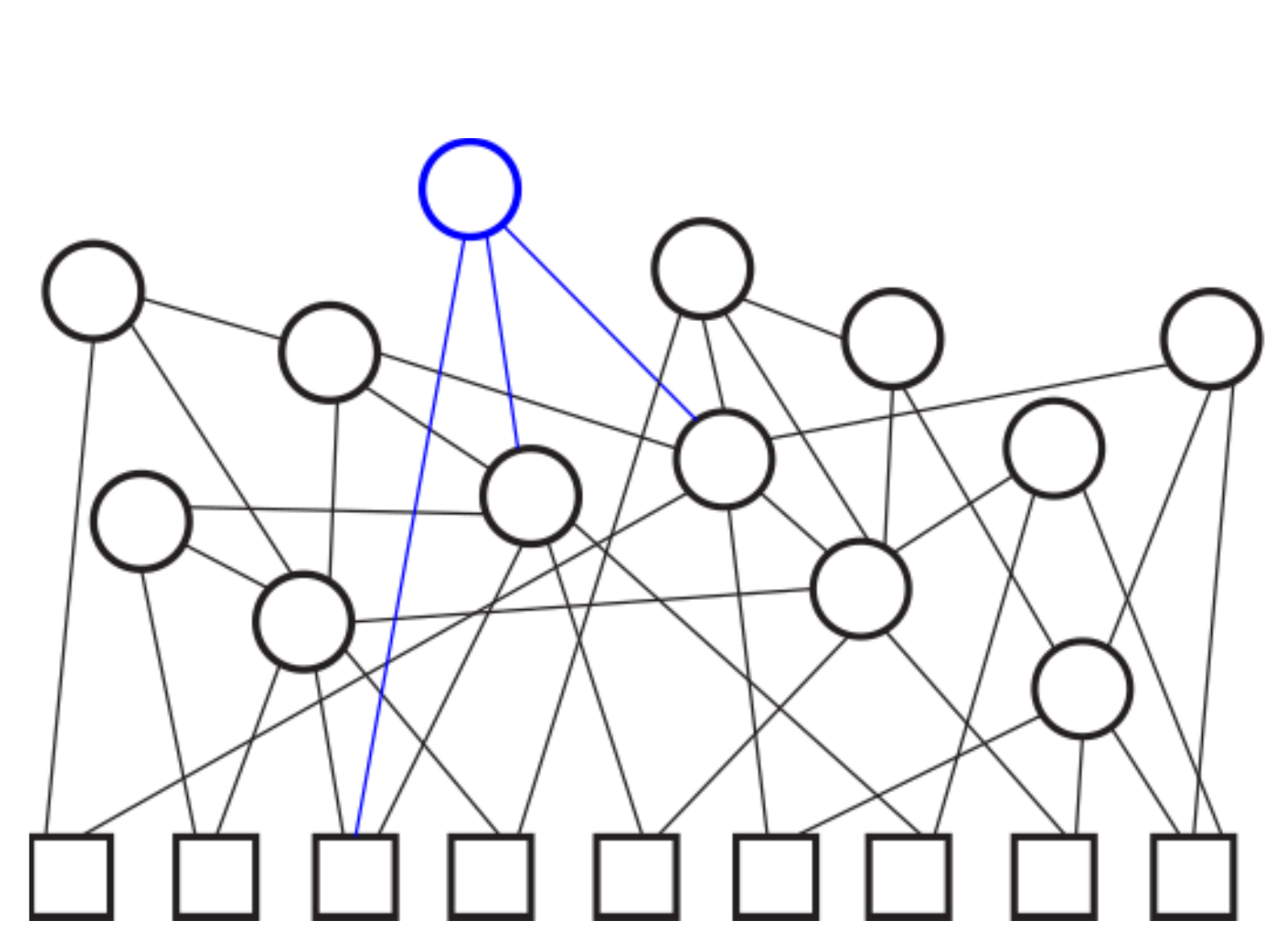}
\quad
\includegraphics[scale=.35]{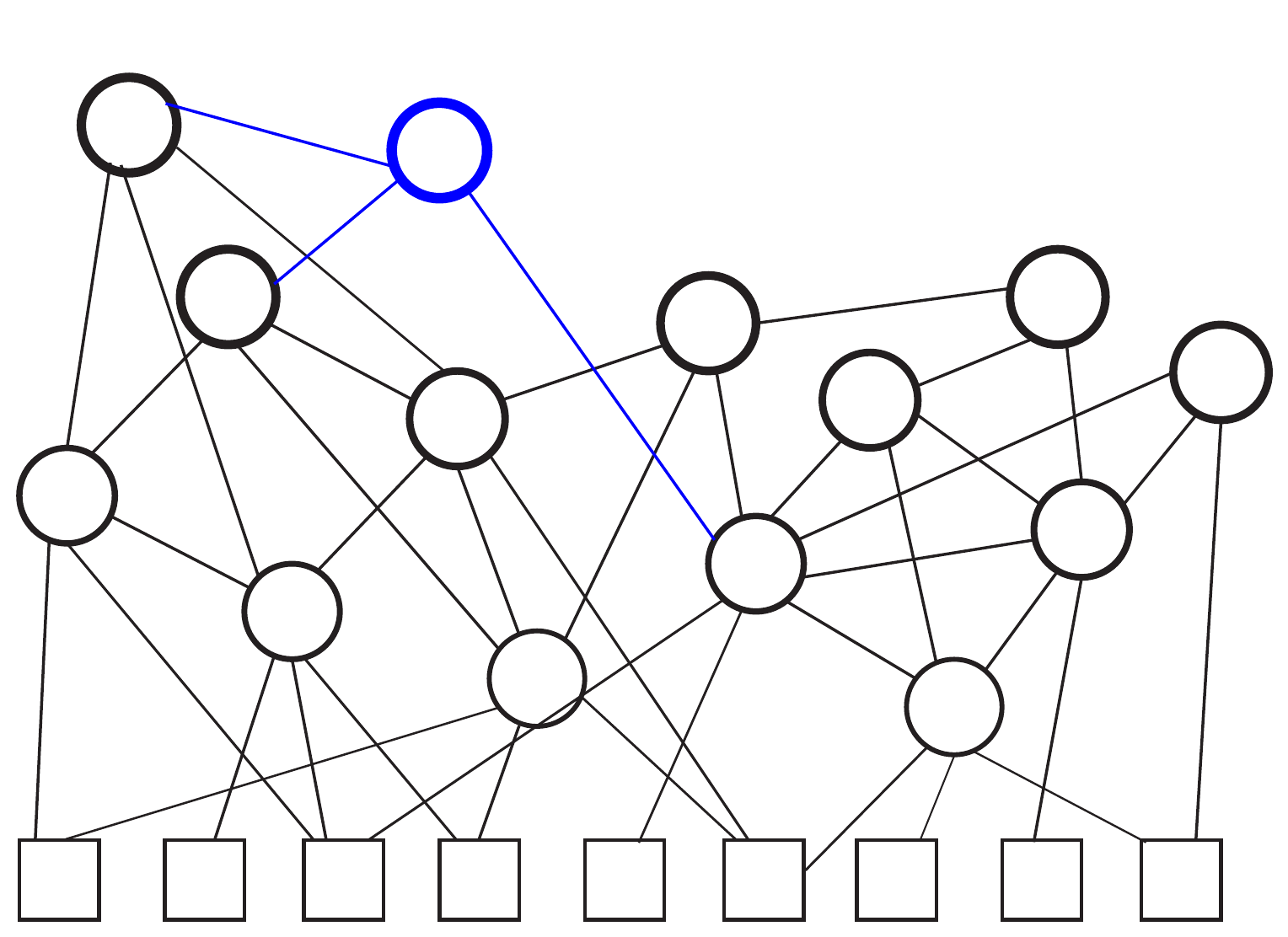}
\caption{Left: an iterative construction. Middle: a wild construction. Right: an exponential construction. In the latter two images, the thickness of the outline representing each item indicates the probability the item will be selected in the construction of the next item.}\label{wild and exp}
\end{figure}

\subsection{Results}

We are interested in the functions computed by high-level items of iterative constructions. In particular, we design iterative constructions so that high-level items compute a threshold function with high probability. 
\begin{defn} The function $f:[0,1] \to [0,1]$ is a $t$-threshold if $f(x)=0$ for $x< t$ and $f(x)=1$ for $x>t$. \end{defn}
For given probability distribution on a set of trees, the output of high-level items of a corresponding iterative construction depends on the following: $(i)$ the fraction of input items firing, $(ii)$ the width of the levels, and $(iii)$ the number of levels. For an $n$ item input, the fraction of input items firing must take the form $ k/n$, $k \in \mathbb{Z}$. Throughout the paper, we assume that the distance between the desired threshold and the fraction of input items firing is at least  $1/n$. To address $(ii)$, we first analyze the functions computed by high-level items of an iterative construction when the width of the levels is infinite, which is equivalent to the ``top down" approach. Then in Section \ref{finite},  we remove this assumption and analyze the ``bottom-up" construction in which the items at level $j-1$ are fixed before the items at level $j$ are created.  The following theorems give a guarantee on the probability that an iterative tree with infinite width levels accurately computes a threshold function in terms of the number of levels. 
To start, we restate Valiant's result \citep{Valiant84a}.
Here $\phi = (\sqrt{5}+1)/2$ is the golden ratio ($2-\phi \approx 0.38$).

\begin{theorem}\label{thm:Valiant}
Let $R$ be the tree that computes $(A \vee B ) \wedge (C \vee D)$. Then, an item at level $\Omega(\log n + \log k)$ of an infinite width iteratively constructed tree for $R$ computes a $(2-\phi)$-threshold function accurately with probability at least $1- 2^{-k}$.
\end{theorem}

In this construction, the iterative tree that computes the $2 - \phi$ threshold function is built using only one small tree. We show that it is possible to achieve arbitrary threshold functions if we allow our iterative tree to be built according to a probability distribution on two distinct smaller trees. 


\begin{theorem}\label{thm:linear}
Let $0<t<1$ and let $R = \{\Pr(T_1)=t, \Pr(T_2)=1-t \}$ where $T_1$ is the tree that computes $(A\vee B ) \wedge C$ and $T_2$ is the tree that computes $(A \wedge B ) \vee C$. Then, an item at level $\Omega( \log n +k)$ of an infinite width iteratively constructed tree for $R$ computes a $t$-threshold function accurately with probability at least $1- 2^{-k}.$
\end{theorem}

The rate of convergence of this more general construction is linear rather than quadratic. While both are interesting, the latter allows us to guarantee a correct function on every input with depth only $O(\log n)$, since there are $2^n$ possible inputs.

\begin{defn} A construction exhibits \textit{linear convergence} if items at level $\Omega( \log n +k)$ of an infinite width iterative tree accurately compute the threshold function with probability at least $1-2^{-k}$. A construction exhibits \textit{quadratic convergence} if items at level $\Omega( \log n +\log k)$ of an infinite width iterative tree accurately compute the threshold function with probability at least $1-2^{-k}$.\end{defn}

The next theorem gives constructions using slightly larger trees with $4$ and $5$ leaves respectively (illustrated in Figure \ref{quad trees five}) that converge quadratically to a $t$-threshold function for a range of values of $t$, with more leaves giving a larger range. Moreover, these ranges are tight, i.e. no construction on trees with $4$ or $5$ leaves yields quadratic convergence to a $t$-threshold function for $t$ outside these ranges. 

\begin{figure}[h]
\centering
\includegraphics[scale=.65]{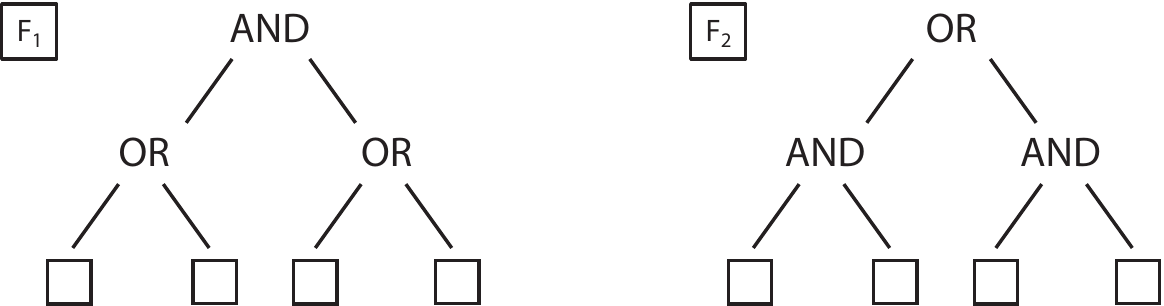}

\bigskip

\includegraphics[scale=.75]{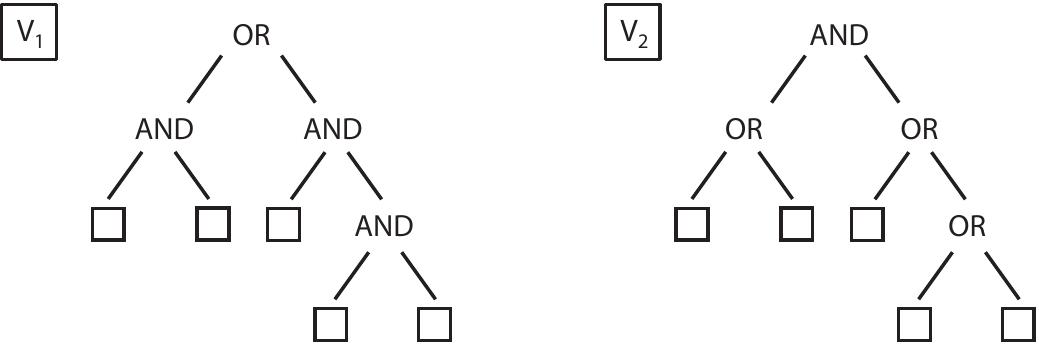}
\caption{For $.38 \lesssim t \lesssim 0.62$, there exists a probability distribution on  $F_1$ and $F_2$ that yields an iterative tree that converges quadratically to a $t$-threshold function. For $.26 \lesssim t \lesssim 0.74$, there exists a probability distribution on $V_1$ and $V_2$ that yields an iterative tree that converges quadratically to a $t$-threshold function. }\label{quad trees five}
\end{figure}


\begin{theorem}\label{thm:quadratic four}

(A) Let $2- \phi\leq t\leq \phi-1$ 
and $\alpha(t)= \frac{1 - t - t^2}{2 t(t-1)}$. Let $R = \{\Pr(F_1)= \alpha(t),   \Pr(F_2)=1-\alpha(t) \}$ be the probably distribution on trees in Figure \ref{quad trees five}. Then, an item at level $\Omega(\log n + \log k)$ of an infinite width iteratively constructed tree for $R$ computes a $t$-threshold function accurately with probability at least $1- 2^{-k}$. Moreover, for $t$ outside this range, there exists no such construction on trees with four leaves that converge quadratically to a $t$-threshold function.\\

(B) Let $\alpha(t)= \frac{-1 + 5 t - 4 t^2 + t^3}{5 t(t-1)}$ and let $t$ be a value for which $0\leq \alpha(t) \leq 1$, so $0.26 \lesssim t \lesssim 0.74$. Let $R = \{\Pr(V_1)= \alpha(t),   \Pr(V_2)=1-\alpha(t) \}$ be the probably distribution on trees in Figure \ref{quad trees five}. Then, an item at level $\Omega(\log n + \log k)$ of an infinite width iteratively constructed tree for $R$ computes a $t$-threshold function accurately with probability at least $1- 2^{-k}$ . Moreover, for $t$ outside this range, there exists no such construction on trees with five leaves that converges quadratically to a $t$-threshold function.
 \end{theorem}


As the desired threshold $t$ approaches $0$ or $1$, we show that an iterative tree that computes the $t$-threshold function must use increasingly large trees as building blocks.  


\begin{theorem}\label{thm:degreenew} Let $t$ be a threshold, $0< t< 1$ and let $s= \min\{t, 1-t\}$. Then, the construction of an iterative tree whose level $\Omega(\log n+ \log k)$ items compute a $t$-threshold function with probability at least $1-2^{-k}$ must be defined over a probability distribution on trees with at least $\frac{1}{\sqrt{2s}}$ leaves. \end{theorem}

This raises the question of whether it is possible to have quadratic convergence for any threshold. We can extend the constructions described in Theorem \ref{thm:quadratic four} by using analogous trees with six and seven leaves to obtain quadratic convergence for thresholds in the ranges $0.15 \lesssim t \lesssim 0.85$ and $0.11 \lesssim t \lesssim 0.89$ respectively. However, it is not possible to generalize this construction beyond this point, as we discuss in Section \ref{steps}. Instead, to achieve quadratic convergence for thresholds near the boundaries, we turn to the following construction, which asymptotically matches the lower bound of Theorem \ref{thm:degreenew}.  We define $A_k$ as a tree on $2k$ leaves that computes $(x_1 \vee x_2 \vee \dots  \vee x_k) \wedge (x_{k+1} \vee x_{k+2} \vee \dots  \vee x_{2k})$ and $B_k$ as a tree on $2k$ leaves that computes $(x_1 \wedge x_2 \wedge \dots  \wedge x_k) \vee (x_{k+1} \wedge x_{k+2} \wedge \dots  \wedge x_{2k})$.

\begin{theorem}\label{thm:k leaves}  For any $0<t \leq 2-\phi$, there exists $k$ and a probability distribution on $A_k$ and $A_{k+1}$ that yields an iterative tree with quadratic convergence to a $t$-threshold function. Similarly for any $ \phi-1  \leq  t<1$, there exists $k$ and a probability distribution on $B_k$ and $B_{k+1}$ that yields an iterative tree with quadratic convergence to a $t$-threshold function. \end{theorem}

There is a trade-off between constructing iterative trees that converge faster and requiring minimal coordination in order to build the subtrees. Building a specified tree on a small number of leaves requires less coordination than building a specified tree on many leaves. Therefore, as $t$ approaches $0$ or $1$, constructing an iterative tree with quadratic convergence becomes less neurally plausible because the construction of each subtree requires much coordination. These results are in line with behavioral findings \citep{Rosch76, Rosch78} and computational models \citep{ArriagaV06, ArriagaRCV15} about categorization being easier when concepts are more robust.

In Section \ref{steps}, we characterize the class of functions that can be achieved by iterative constructions allowing building block trees of any size. We show that it is possible to achieve an arbitrarily close approximation of any staircase function in which each step intersects the line $y=x$. This result is described more precisely in Theorem \ref{thm:staircase}.


In the following section we turn to finite realizations of iterative trees. The above theorems analyze the behavior of an iterative construction where the width of the levels is infinite. We assumed that for any input the number of items turned on at given level of the tree is equal to its expectation. Imagining a ``bottom up" construction, we note that the chance that the number of items firing at a given level deviates from expectation is non-trivial. Such deviations percolate up the tree and effect the probability that high-level items compute the threshold function accurately. The smaller the width of a level, the more likely that the number of items on at that level deviates significantly from expectation, rendering the tree less accurate. 
How large do the levels of an iteratively constructed tree need to be in order to ensure a reasonable degree of accuracy? 


\begin{theorem} Consider a construction of a $t$-threshold function with quadratic convergence described in Theorem \ref{thm:quadratic four} or Theorem \ref{thm:k leaves} in which each level $\ell$ has $m_\ell$ items and the fraction of input items firing is at least $\ve$ from the threshold $t$. Then, with probability at least $1-\gamma$,  items at level $\Omega\left( \log{\frac{1}{\gamma}} + \log{\frac{1}{\ve}}\right)$ will accurately compute the threshold function for $m_1=\Omega\left(\frac{ \ln(1/\gamma) }{\ve^2}\right)$ and $\sum_{\ell}m_\ell = O(m_1)$. \label{finite width quad} \end{theorem}

As a direct corollary, by setting $\ve = O(1/n)$ and $\gamma = 2^{-n-1}$, we realize a $t$-threshold construction of size $O(n^3)$ for any $t$, matching the best-known construction which was for a specific threshold \citep{Hoory2006}. 
The finite-width version of Theorem \ref{thm:linear} is given in Section \ref{finite}.

The exponential iterative construction also converges to a $t$-threshold function for appropriate $\alpha$. We give the statements here for the wild iterative construction (with no weight decay) and the general exponential construction.

\begin{theorem} Consider a wild construction on $n$ inputs for the $t$-threshold function given in Theorem \ref{thm:linear} in which $n > \log( \frac{1}{8\eps \delta}) \max\{\frac{1}{\ve^2},\frac{1}{\delta^2}\}$ where $\ve$ is the distance between $t$ and the fraction of inputs firing. Then, there is an absolute constant $c$ such that for $k= \Omega\left(n \left(\frac{1}{\delta^c}+\frac{1}{\ve^c}\right)\right)$, the $k^{th}$ item accurately computes the $t$-threshold function with probability at least $1- \delta$. \label{wild} \end{theorem}

\begin{theorem} Consider an exponential construction on $n$ inputs for the $t$-threshold function given in Theorem \ref{thm:linear} in which $\alpha = \lfloor \frac{\min \{\ve^2, \delta^2\}}{2048e^4\log(4/\ve\delta)}\rfloor$ and $n > 1/\alpha.$ Then for $$k = \Omega\left(\frac{\log \frac{1}{\eps\delta}}{\min \{\eps^2, \delta^2\}}\left(\log \frac{n}{\eps\delta}\right)\right),$$ with probability at least $1-\delta$, the $k^{th}$ item will compute the $t$-threshold function. \label{exponential} \end{theorem}

Finally, in Section \ref{learning}, we give a simple cortical algorithm to learn a uniform threshold function from a single example, described more precisely by the following theorem. 
\begin{theorem}\label{thm:learning} 
Let $X \in \{0,1\}^n$ such that $||X||_1=tn$, $L= \Omega\left( \log{\frac{1}{\gamma}} + \log{\frac{1}{\ve}}\right)$, and $\ve=\Omega\left(\sqrt{\frac{\ln\left(1/\gamma \right)}{m}}\right)$. Then, on any input in which the fraction of input items firing is outside $[t-\ve, t+\ve]$, items at level $L$ of an iterative tree produced by LearnThreshold($L,m,X$) will compute a $t$-threshold function with probability at least $1-\gamma$. \end{theorem}

The next section provides the groundwork for these theorems, and the proofs are in Sections \ref{convergence}, \ref{steps}, \ref{finite}, \ref{learning}. We discuss several open questions and directions for future research in Section \ref{sec:discuss}.

\section{Polynomials of AND/OR Trees}

Let  $g_T: \{0, 1\}^n \to \{0,1\}$ be the Boolean function computed by an AND/OR tree $T$ with $n$ leaves. 
We define $f_T$ as the probability that $T$ evaluates to $1$ if each input item is independently set to $1$ with probability $p$.
$$f_T(p)=\Pr\left(g_T(X)=1 \, | \, X\sim B(n,p)\right). $$

\killtext{
\begin{figure}[h]
\centering
\includegraphics[scale=.75]{andortrees.eps}
\caption{Left: $f_T(p)=2p-p^2.$ Middle: $f_T(p)=p^2$.  Right: $f_T(p)=4p^2-4p^3+p^4$.  }\label{examples}
\end{figure}
}

We analogously define $f_C(p)$ for probability distributions on trees; let $f_C$ be the probability that a tree chosen according to $C$ evaluates to 1 if each input item is independently set to $1$ with probability $p$. Let $\lambda_T$ be the probability of $T$ in distribution $C$. We have

$$f_C(p)=\sum_{T \in C} \lambda_T f_T(p) .$$ 

In an iterative construction for the probability distribution $C$, an item at level $k$ evaluates to 1 with probability $f_C(p_{k-1})$ where $p_{k-1}$ is the probability that an item at level $k-1$ evaluates to $1$. In the case where the width of the levels is infinite, the fraction of inputs firing any level is exactly equal its expectation. Therefore, the probability that items at level $k$ evaluate to 1 is $f_C^{(k)}(p)$ where $p$ is the probability an input is set to 1. This follows directly from the recurrence relation: $$ f_C^{(k)}(p)=f_C(f_C^{(k-1)}(p)).$$ 
 
In the remainder of this section, we collect properties of polynomials of AND/OR trees to be used in the analysis of iterative trees.

 We call a polynomial {\em achievable} if it can be written as $f_T$ for some AND/OR tree $T$. We call a polynomial {\em achievable through convex combinations } if it can be written as $f_C$ for some probability distribution on AND/OR trees $C$. Table \ref{the table} lists all achievable polynomials with degree at most five. Note that $\mathcal{A}$ is closed under the AND and OR operations. If $a, b \in \mathcal{A}$, then  $a \cdot b \in \mathcal{A}$ and $a+ b -a \cdot b\in  \mathcal{A}$. The set of polynomials achievable through convex combinations is the convex hull of $\mathcal{A}$.

\begin{lemma} \label{achievable} Let $\mathcal{A}$ be the set of achievable polynomials. Let $A(x)=a_0 + a_1 x + \dots + a_n x^n$ be a polynomial in $\mathcal{A}$. Then,
\begin{enumerate}
 \item $a_0=0$
\item $a_n=-1 \text{ or } 1$
\item $\sum_{i=0}^n a_i=1$
\item If $A(x)$ has degree $d$, then $A(x)$ is the polynomial for a tree on $d$ leaves. 
 \end{enumerate}
\end{lemma}

\begin{proof} We proceed by induction on the degree of $A(x)$. For $d=1$, $A(x)=x$ is the only polynomial in $\mathcal{A}$ and all the above properties hold. Next assume all the properties hold for polynomials of degree less than $d$. Let $A(x)$ be an achievable polynomial of degree $d$. Then the root of the tree for $A$, which we call $T_A$, is either an AND or an OR operation. In the former case,  $A= B \cdot C$ and in the latter case  $A= B+C-B \cdot C$ where $B,C \in \mathcal{A}$ and $B$ has degree $k$ and $C$ has degree $d-k$ for $0< k< d$. In either case, the first three properties follow trivially from the inductive hypothesis. For item (4), let $T_B$ and $T_C$ be trees that correspond to $B$ and $C$ respectively. Then $T_B$ and $T_C$ have $k$ and $d-k$ leaves respectively. Since $T_A$ is $T_B$ adjoined with $T_C$ with an  AND or OR operation, $T_A$ has $d$ leaves. \end{proof}



\begin{lemma} Let $f \in \mathcal{A}$ be an achievable polynomial of degree $d$, $f=a_0 + a_1 x+ a_2 x^2+ \dots a_dx^d$. Then $ |a_\ell | \leq d^\ell$.\label{coefficients} \end{lemma}

\begin{proof} Proceed by induction. The only achievable polynomial of degree $1$ is $f(x)=x$, so the statement clearly holds. Next, assume $|a_{\ell'} | \leq d^{\ell'}$ holds for all $l'<l$. Let $f$ be a degree $d$ achievable polynomial.  We may assume $f= g + h- gh$ or $f= gh$ where $g$ and $h$ are achievable polynomials with degree $k$ and $d-k$ respectively where $k \leq \frac{\ell}{2}$.  First consider the case when $f= g+h-gh$, meaning the root of the tree corresponding to $f$ is an OR operation. Observe
 \begin{align*}
|a_\ell(f)|&=\big| a_\ell(g)+ a_\ell(h)- \sum_{i=1}^{l-1} a_i(g) a_{l-i}(h)\big|\\
&\leq  k^\ell +(d-k)^\ell +\sum_{i=1}^{l-1} k^i (d-k)^{l-i}\\
&\leq ((d-k)+k)^\ell\\
&=d^\ell.
\end{align*}

Next consider the case when $f= gh$, meaning the root of the tree corresponding to $f$ is an AND operation. Observe  that 
\[
|a_\ell(f)| =\big| \sum_{i=1}^{l-1} a_i(g) a_{l-i}(h)\big| \leq \sum_{i=1}^{l-1} k^i (d-k)^{l-i}
<d^\ell.
\]
\end{proof}

We observe a relationship between the polynomial of a tree and the polynomial of its complement. We define the complement of the AND/OR tree $T$ to be the tree obtained from $T$ by switching the operation at each node. 

\begin{lemma} \label{complement} Let $A$ and $B$ be complementary AND/OR trees and let $f_A$ and $f_B$ be the corresponding polynomials. Then $f_B(1-p)=1-f_A(p)$ for all $0<p<1$. \end{lemma}

\begin{proof} Proceed by induction on the number of leaves of the tree. For a tree on one leaf, the statement holds trivially. Without loss of generality, assume that the root of tree $A$ is an AND operation. Then $f_A(x)=a_1(x)a_2(x)$ and $f_B(x)=b_1(x)+b_2(x)-b_1(x)b_2(x)$ where the trees corresponding to $a_1$ and $b_1$ are complements and the trees corresponding to $a_2$ and $b_2$ are also complements. By the inductive hypothesis, $a_1(p)=1-b_1(1-p)$ and $a_2(p)=1-b_2(1-p)$. Observe \begin{align*}
1-f_A(p)&=1-a_1(p)a_2(p)\\
&= 1-(1-b_1(1-p))(1-b_2(1-p))\\
&=b_1(1-p)+b_2(1-p)-b_1(1-p)b_2(1-p)\\
&= f_B(1-p). \end{align*}
 \end{proof}

Let $f_A$ be a polynomial achievable through convex combinations, $f_A=\sum_{i=1}^n \lambda_i f_{A_i}$. Let $A_i$ and $B_i$ be complementary AND/OR trees. Let $f_B= \sum_{i=1}^n \lambda_i f_{B_i}$. We say that $f_A$ and $f_B$ are \textit{complementary polynomials}. 

\begin{corollary}  Let $f_A$ and $f_B$ be complementary polynomials. Then \begin{enumerate}
\item For all $0<p<1$, $f_B(1-p)=1-f_A(p)$
\item If $p$ is a fixed point of $f_A$ then $1-p$ is a fixed point of $f_B$
\item For all $0<p<1$, $f_B^{(k)}(1-p)=1-f_A^{(k)}(p).$
\end{enumerate}\label{cor:complementary}
\end{corollary}

\begin{defn} We say that $t$ in an {\em attractive} fixed point of $f$ if there exists $\ve>0$ such that for all $x \in [0,1]$ such that $|x-t| <\ve$, $|t-f(f(x))|< |t-f(x)|$. We say that $t$ in a {\em non-attractive} fixed point of $f$ if there exists $\ve>0$ such that for all $x \in [0,1]$ such that $|x-t| <\ve$, $|t-f(f(x))|> |t-f(x)|$. Equivalently, $t$ is a non-attractive fixed point if $f'(t)>1$, and an attractive fixed point if $f'(t)< 1$.\end{defn}
For the function illustrated in Figure \ref{graph}, 0 and 1 are attractive fixed points and $1/2$ is a non-attractive fixed point. For the function illustrated in Figure \ref{onestep}, 0 and 1 are non-attractive fixed points and $1/2$ is an attractive fixed point.

\begin{lemma} Let $f \in \mathcal{A}$ be an achievable polynomial corresponding to a tree $T$ with at least one AND or OR operation. Then\begin{enumerate}
\item The function $f$ has an attractive fixed point at 0,  a non-attractive fixed point at 1, and no  fixed points in $(0,1)$ if and only if there is a path from the root to a leaf in $T$ in which each node represents an AND operation. 
\item The function $f$ has an non-attractive fixed point at 0,  an attractive fixed point at 1, and no fixed points in $(0,1)$ if and only if there is a path from the root to a leaf in $T$ in which each node represents an OR operation.
\item The function $f$ has attractive fixed points at 0 and 1,  and precisely one non-attractive fixed point $\alpha$ in $(0,1)$ if and only if there is no path from the root to a leaf in $T$ in which each node represents an AND operation and there is no path from the root to a leaf in $T$ in which each node represents an OR operation. Moreover, $\alpha$ is irrational.
\end{enumerate} \label{lemma paths}
\end{lemma}
\citep{Moore1956} prove a stronger version of (3). They show that any polynomial  $f$ corresponding to an arbitrary circuit of AND and OR operations can have at most one fixed point on $(0,1)$. We present a similar version of their argument for the setting when $f$ is corresponds to a tree. 

\begin{proof} It suffices to prove necessity for each statement.

(1) Let $p^2=a_1, \dots, a_k=f$ be the polynomials computed by the nodes of some AND path. Since each $a_{i}$ corresponds to a tree with an AND root, $a_i= a_{i-1} g_i$ where $g_i$ is the polynomial corresponding to the subtree of the node that does not intersect that AND path. Thus, $f=a_k = p^2 \prod_{i=2}^k g_i$. For $0<p<1$, $$\frac{f(p)}{p} =p  \prod_{i=2}^k g_i(p)< p<1.$$ Therefore, $f$ has no fixed points on $(0,1)$. To see that $0$ is an attractive fixed point, note that $f(p)<p$ for any $p \in (0,1)$, Thus, $f(f(p))< f(p)$. Similarly,  $1$ is a non-attractive fixed point because $1-f(f(p))> 1-f(p)$ for all $p \in (0,1)$.
 
(2) Follows from (1) and Lemma \ref{complement}. 

(3) First we show that if the tree has no path of OR operations, then the corresponding polynomial $f$ will not have a linear term by proving the contrapositive. The root of a tree corresponding to a polynomial with a linear term must be an OR operation since if the root were an AND operation, the corresponding polynomial would the product of two non-zero polynomials with no constant terms. Given that the root has a linear term and computes $g+h -gh$ for some achievable polynomials $g$ and $h$, it follows that  $g$ or $h$ has a linear term. We iteratively apply this argument for the appropriate subtrees and conclude that there is a path of OR operations. 

Since $f$ has no linear term, for small $\ve$, $f(\ve)= O(\ve^2)$. Therefore, there is an $\ve$ neighborhood around zero such that $f(p)<p$. It follows that $f(f(p))< f(p)$, so $0$ is an attractive fixed point. The fact that 1 is also an attractive fixed point follows from Lemma \ref{complement}. Since there is no AND path for $T$, there is no OR path in the complementary tree. Thus, 0 is an attractive fixed point of $f^c$, so 1 is an attractive fixed point of $f$. 

Next, we show that there exists some fixed point of $f$ on $(0,1)$. Since $0$ and $1$ are attractive fixed points, there exists $\ve_1,\ve_2>0$ such that $f(\ve_1)<\ve_1$ and $f(1-\ve_2)> 1-\ve_2$. By the intermediate value theorem, $f$ must cross the line $y=x$. Thus, there exists some $\alpha \in ( \ve_1, 1-\ve_2) \subset (0,1)$ such that $f(\alpha)=\alpha$. 

To prove that $\alpha $ is a non-attractive fixed point and that $\alpha$ is the unique fixed point on $(0,1)$, it suffices to show that for any fixed point $\alpha \in (0,1)$,  $f'(\alpha)\geq 1$.  We use an argument inspired by  \citep{Moore1956} to prove  $f'(p)\geq \frac{f(p)(1-f(p))}{p(1-p)}$ for $p \in (0,1)$. 
Proceed by induction on the size of the tree. Clearly, the statement holds for the leaves which compute the polynomial $p$. Let $g, h$ be the polynomials associated to the two subtrees joined at the root of $T$. If the root is an AND operation, $f=gh$, and we have $$f'= gh'+g'h \geq \frac{gh(1-h)}{p(1-p) } + \frac{hg(1-g)}{p(1-p) } = \frac{gh(2-g-h)}{p(1-p)}\geq \frac{gh (1-gh)}{p(1-p)}=\frac{f (1-f)}{p(1-p)}.$$ The final inequality uses the fact that since $(1-g)(1-h) \geq 0$, $ 2-g-h \geq 1-gh$. If the root is an OR operation, $f= g+h-gh$ and we have $$f'= g'+h' -gh'-g'h =h'(1-g)+g'(1-h)\geq \frac{h(1-h)(1-g)+g(1-g)(1-h)}{p(1-p) } $$ 
$$ = \frac{(1-h-g+gh)(g+h)}{p(1-p)}\geq \frac{(1-h-g+gh)(g+h-gh)}{p(1-p)}= \frac{f(1-f)}{p(1-p)}.$$

Finally, we prove that $\alpha$ is irrational. 
Suppose for contradiction that $\alpha= m/n$ where $m, n \in \Z$ and are relatively prime. Let $k+3$ be the degree of $f$.
We may write $$f(p)= g(p) (1-p)\left(p-\frac{m}{n}\right)p= \sum_{i=0}^k a_i p^{i+1} \left( -p^2 +
\left( 1+ \frac{m}{n} \right) p -\frac{m}{n}\right),$$ where $g(p)= \sum_{i=0}^k a_i p^i$ and
 $a_k \not=0$. Let $b_i$ be the coefficient of the term $p^i$ in $f$. Since $f$ is achievable, 
 by Lemma \ref{achievable}, each $b_i \in \Z$, $b_0=0$, and $b_{k+3}$ is $1$ or $-1$. 
 We show by induction that for all $i$ $a_i=nt_i$ for some $t_i \in \Z$.  First note that $b_1= \frac{-m}{n} a_0$. Since $b_1 \in \Z$, $a_0=nt_0$ where $t_0 \in \Z$. 
 Similarly, $b_2= \left( 1+ \frac{m}{n} \right) a_0 + \left( \frac{-m}{n}\right) a_1$, 
 so $a_1=nt_1$ for some $t_1 \in \Z$. We have $b_3= -a_{0} +\left(1+\frac{m}{n}\right)  
 a_{1} + \left( \frac{-m}{n} \right)a_{2}$, so $a_2=nt_2$ for some $t_1 \in \Z$. 
 Assume that $a_\ell= n t_\ell$ for all $\ell<j-1$. We have
$$b_j= -a_{j-3} +\left(1+\frac{m}{n}\right)  a_{j-2} + \left( \frac{-m}{n} \right)a_{j-1}.$$ 
Since $a_{j-3}$ and $a_{j-2}$ are integer multiples of $n$, it follows that $ -a_{j-3} +\left(1+\frac{m}{n}\right)  a_{j-2} $ is integer. Since $b_j$ is integer, $ a_{j-1}= t_{j-1} n$ for some $t_{j-1} \in \Z$. We have shown that $a_0, \dots a_k$ are integer multiples of $n$. We have $b_{k+3}= -a_{k}.$ Since $a_k$ is a non-zero integer multiple of $n$, it follows that $b_{k+3}$ is not -1 or 1, a contradiction. 
\end{proof}

Finally, we make some observations about the polynomials associated with the specific family of trees we use in many of our constructions. 

\begin{defn}Let $A_k$ be a tree on $2k$ leaves that computes $(x_1 \vee x_2 \vee \dots  \vee x_k) \wedge (x_{k+1} \vee x_{k+2} \vee \dots  \vee x_{2k})$. Let $B_k$ be a tree on $2k$ leaves that computes $(x_1 \wedge x_2 \wedge \dots  \wedge x_k) \vee (x_{k+1} \wedge x_{k+2} \wedge \dots  \wedge x_{2k})$.
\end{defn}

\begin{lemma}Let $f_{A_k}$ and $f_{B_k}$ be the polynomials corresponding to $A_k$ and $B_k$ respectively. Then $f_{A_k}$ has a unique fixed point in the interval $\left(\frac{1}{k^2}, \frac{1}{k(k-1)}\right)$
and $f_{B_k}$ has a fixed point in the interval $\left(1-\frac{1}{k(k-1)}, 1-\frac{1}{k^2}\right)$.
\label{root range} \end{lemma}
The proof of Lemma \ref{root range} will use the following elementary inequality.
\begin{lemma}\label{1-x-approx}
For $x \in (0,1)$,  and any integer $k \ge 0$,
$1- kx < (1-x)^k < 1 - kx + {k \choose 2} x^2$.
\end{lemma}
\begin{proof}[of Lemma \ref{root range}.] It suffices to show that $g(x)=f_{A_k}(x)-x$ has a zero on the interval $\left(\frac{1}{k^2}, \frac{1}{k(k-1)}\right)$.
We will show that $g(1/k^2) < 0$ and $g(1/(k(k-1))) > 0$ and apply the intermediate value theorem.   
Using Lemma \ref{1-x-approx}, for $x=1/k^2$,
\[
(1-(1-x)^k)^2 -x < (1-(1-kx))^2 -x = k^2x^2 -x = 0.
\]
Similarly, for $x = 1/(k(k-1))$,
\begin{align*}
(1-(1-x)^k)^2 -x &> \left(1-\left(1-kx+ \frac{k(k-1)x^2}{2}\right)\right)^2 - x\\ 
&= \left(\frac{1}{k-1}-\frac{1}{2k(k-1)}\right)^2- \frac{1}{k(k-1)}\\
&= \frac{1}{(k-1)^2}\left(1 - \frac{1}{2k}\right)^2 - \frac{1}{k(k-1)}\\ 
&=  \frac{1}{(k-1)^2}\left(1- \frac{1}{k} + \frac{1}{4k^2} - (1 - \frac{1}{k})\right) = \frac{1}{4k^2(k-1)^2} > 0.
\end{align*}
It follows from Corollary \ref{cor:complementary} that $f_{B_k}$ has a fixed point in the interval $\left(1-\frac{1}{k(k-1)}, 1-\frac{1}{k^2}\right)$. Uniqueness follows from Lemma \ref{g}.
\end{proof}

\begin{lemma} Let $0\leq \alpha \leq 1$ and $f=\alpha f_{B_k}+(1-\alpha) f_{B_{k+1}}$ where $f_{B_k}$ is the polynomial corresponding to $B_k$. Let $t$ be the fixed point of $f$ in $(0,1)$. Then $g(p)= \frac{f(p)-p}{p(1-p)(p-t)}\geq \frac{1}{t}$ for all $p \in [0,1]$. 
\label{g} \end{lemma}

\begin{proof}
By definition $$g(p) = \frac{f(p)-p}{p(1-p)(p-t)}=\frac{ \alpha(1-2p^{k-1}+p^{2k-1})+(1-\alpha)(1-2p^k+p^{2k+1})}{(1-p)(t-p)}.$$ Since $1$ and $t$ are fixed points of $f(p)$, $(1-p)$ and $(t-p)$ divide $f(p)-p$. Therefore, we may write $g= a_0 +a_1p  \dots + a_{2k-1} p^{2k-1}$ polynomial. We claim that all coefficients of $g$ are positive. Note that 
\begin{multline*}(t-p)\sum_{i=0}^{2k-1} a_i p^i=\\ \alpha(1+p \dots + p^{k-2}-p^{k-1}- p^{k}  \dots - p^{2k-2})+ (1-\alpha)(1+p \dots + p^{k-1}-p^{k}- p^{k+1}  \dots - p^{2k}).\end{multline*}
Observe \begin{align*} a_i&=\frac{1}{t}&\text{ for } i=0\\
a_i&=\frac{a_{i-1}+1}{t} &\text{ for } 1\leq i\leq k-2\\
a_i&=\frac{a_{i-1}+1-2\alpha }{t} &\text{ for } i=k-1\\
a_i&=\frac{ a_{i-1}-1}{t} &\text{ for } k \leq i\leq 2k-2\\
a_i&=\frac{ a_{i-1}-(1-\alpha)}{t} &\text{ for } i =2k-1.\\
\end{align*}
Note that $\frac{1}{t}=a_0 < a_1< \dots < a_{k-2}$, so $a_i>0$ for all $0\leq i\leq k-2$. Next observe that $a_{2k-1}=1-\alpha$ since comparing the coefficients of the $p^{2k}$ terms on both sides gives $ -a_{2k-1} = -(1-\alpha) $. It follows that $a_{2k-2}=t(1-\alpha)+1-\alpha$. For all $k\leq i\leq 2k-2$, $a_{i-1}=ta_i+1.$ Therefore $a_{i}> 0$ for all $k-1\leq i\leq 2k-3$. 

Since all coefficients of $g$ are positive, all derivatives are increasing. In particular the first derivative of $g(p)$ is increasing on the interval $(0,1)$. Therefore $g(p) \geq g(0) = \frac{1}{t}$ for all $p \in [0,1]$. 
\end{proof}

\begin{lemma} Let $f_{A_k}$ and $f_{B_k}$ be the polynomials corresponding to $A_k$ and $B_k$ respectively. For $t\leq 2-\phi$, there exists some $k$ and $\alpha$ such that $f_A= \alpha f_{A_k}+ (1-\alpha) f_{A_{k+1}}$ has fixed point $t$. Moreover, $\frac{t-f_A(p)}{t-p}\geq \left(1+\frac{p(1-p)}{t}\right)$. Similarly, for $t\geq \phi-1$, there exists some $k$ and $\alpha$ such that $f_B= \alpha f_{B_k}+ (1-\alpha) f_{B_{k+1}}$ has fixed point $t$. Moreover, $\frac{t-f_B(p)}{t-p} \geq \left(1+\frac{p(1-p)}{t}\right)$. \label{k,k+1} \end{lemma}

\begin{proof} By Corollary \ref{cor:complementary}, it suffices to prove the theorem for $\phi -1 \leq t< 1$.  By Lemma \ref{root range}, $f_{B_k}$ has a single fixed point in the range $\big(1-\frac{1}{k(k-1)}, 1- \frac{1}{k^2}\big)$. Let $b_k$ be the fixed point $f_{B_k}$. 
Note that for $0<p<1$, 
$$f_{B_{k+1}}(p)< f_{B_{k}}(p).$$ It follows that $b_k<b_{k+1}$. We obtain an increasing sequence $\phi-1=b_2, b_3, b_4 \dots $ that converges to 1.  Let $k$ be the value for which $b_k\leq t< b_{k+1}$. Let $f= \alpha f_{B_k} + (1-\alpha) f_{B_{k+1}}$ where $\alpha$ is chosen so that $f$ has fixed point $t$.

Next we show that $t-f(p)\geq (t-p)\left(1+\frac{p(1-p)}{t}\right)$. By Lemma \ref{g}, $g(p) = \frac{f(p)-p}{p(1-p)(p-t)} \geq \frac{1}{t}$ for $0\leq p\leq 1$. Therefore \begin{align*} t-f(p)
=(t-p)(1+ p(1-p)g(p))
\geq (t-p)\left(1+\frac{p(1-p)}{t}\right).\end{align*} \end{proof}

\section{Convergence of iterative trees to threshold functions}\label{convergence}

In the previous section, we showed that if the width of each level is infinite, then items at level $k$ of an iterative tree evaluate to 1 with probability $f_C^{(k)}(p)$ when each input is independently set to 1 with probability $p$. In this section, we demonstrate ways of selecting $C$ so that $f_C^{(k)}(p)$ converges to a $t$-threshold function. 

By an abuse of notation, we say that $f(p)$ converges to a $t$-threshold function if
$$
\lim_{k \to \infty}f^{(k)}(p) = \left\{
        \begin{array}{ll}
            0 & \quad 0 \leq p<t  \\
            1 & \quad t<p\leq 1 \\
            p & \quad p=t.
        \end{array}
    \right.
$$
Moreover, we say that $f$ converges quadratically to a $t$-threshold function if the corresponding iterative construction exhibits quadratic convergence. The function depicted in Figure \ref{graph} converges to a $1/2$-threshold function. 
 \begin{figure}[h]
 \centering
\includegraphics[width=.2\textwidth]{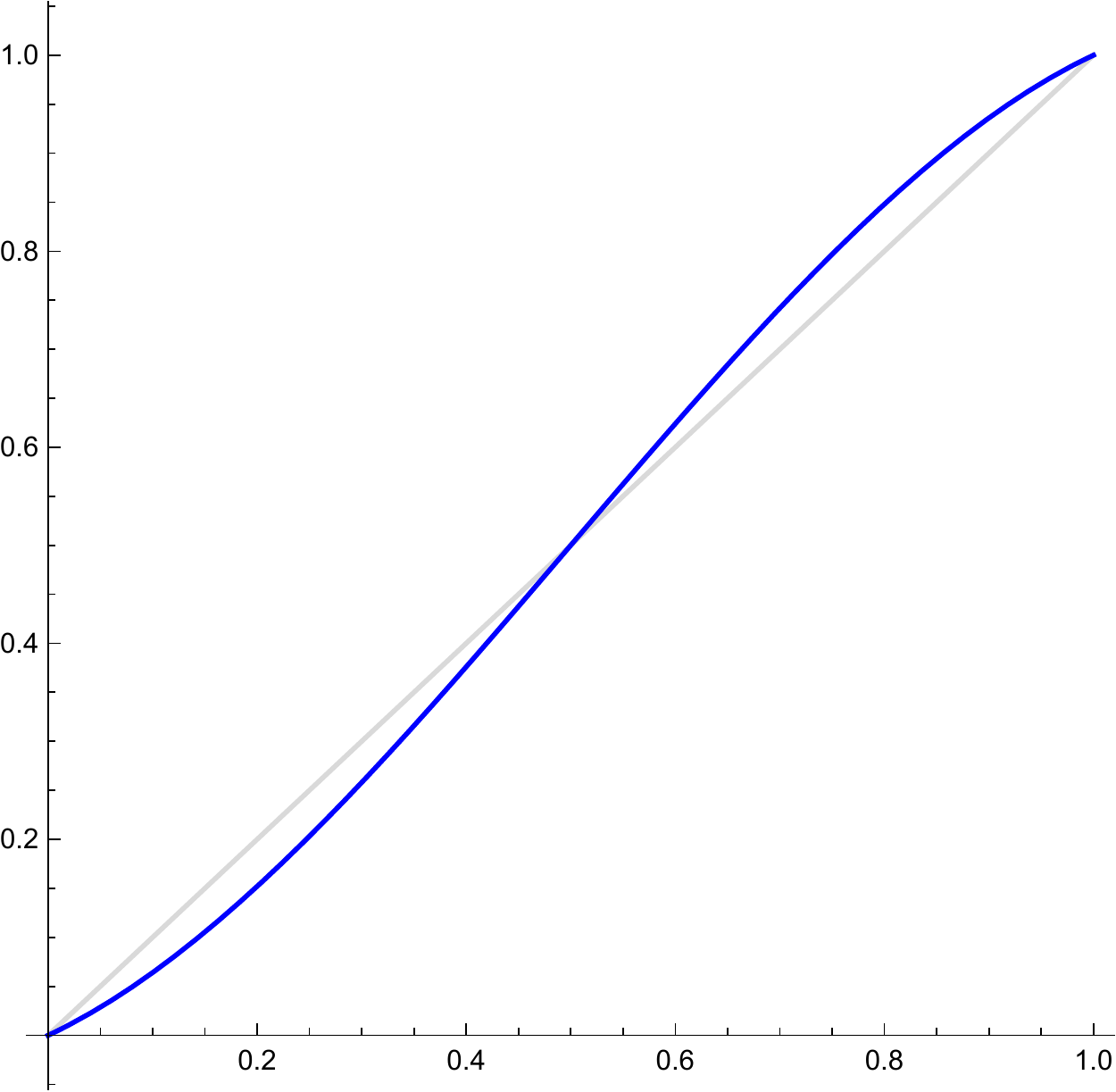}
\quad
\includegraphics[width=.2\textwidth]{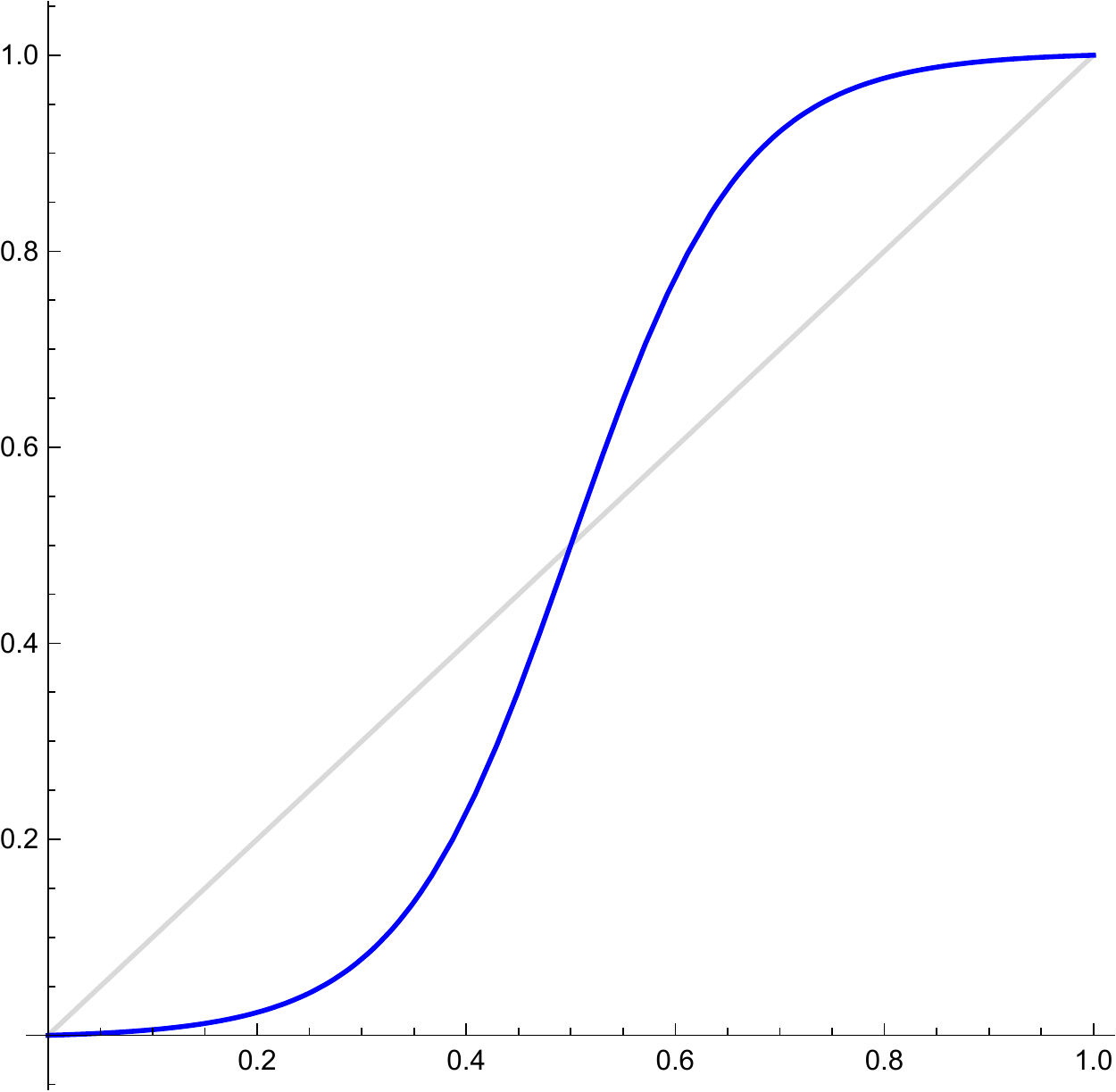}
\quad
\includegraphics[width=.2\textwidth]{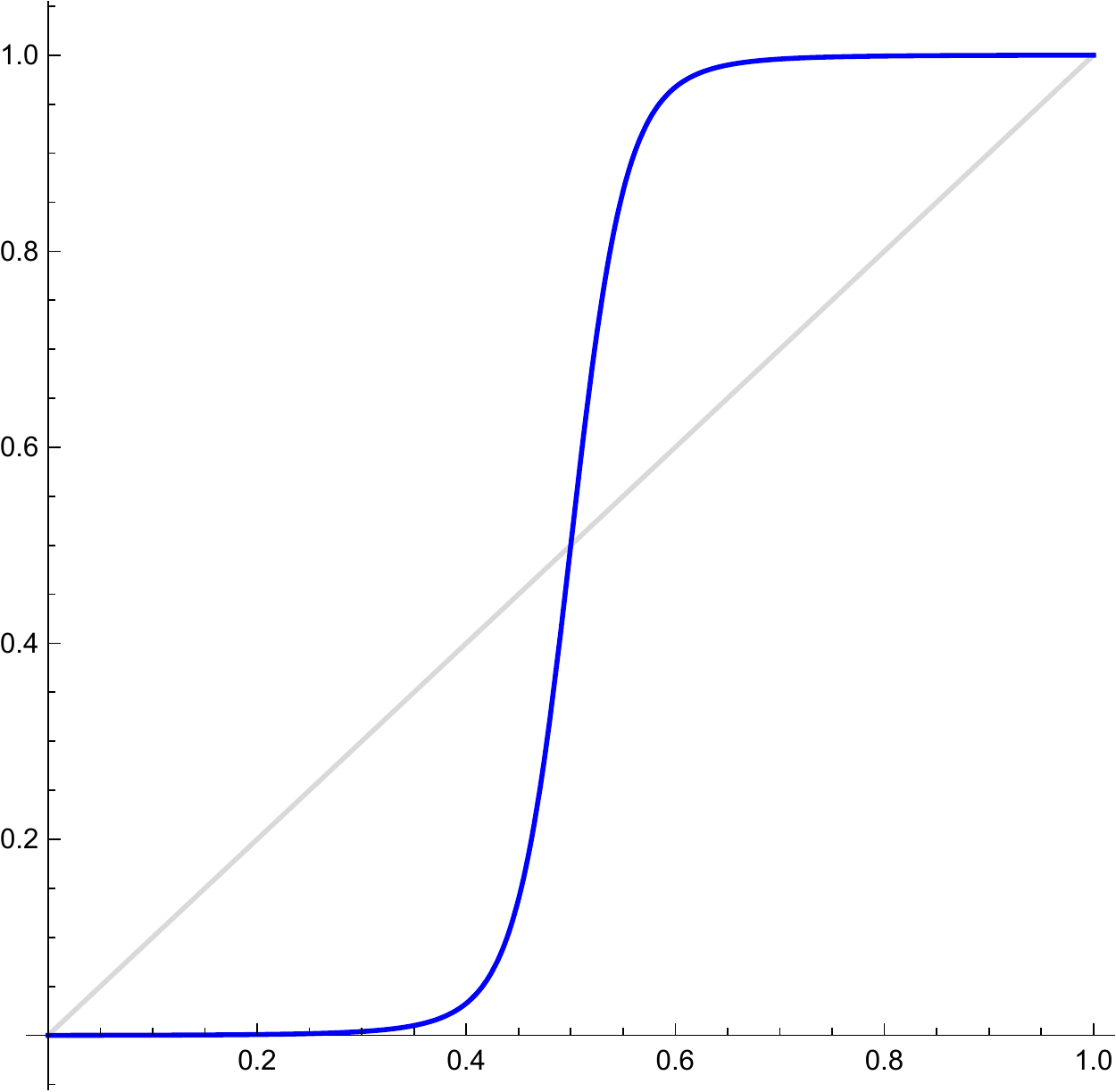}
\quad
\includegraphics[width=.2\textwidth]{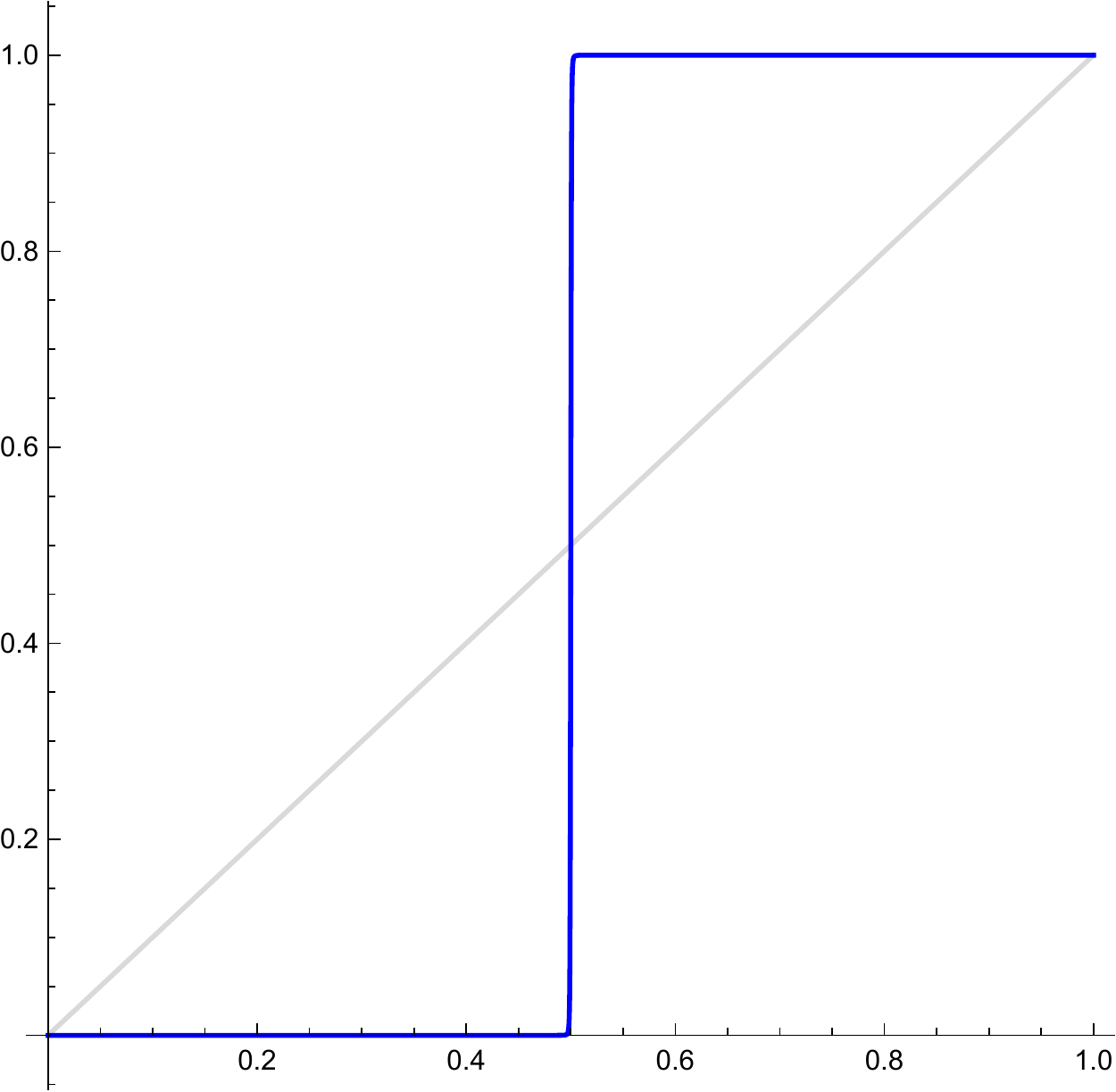}
  \caption{A function $f$ that converges to a $1/2$-threshold function. Left to right: $f(p), f^{(5)}(p),f^{(10)}(p),f^{(30)}(p)$.}
  \label{graph}
  \end{figure}

We now prove that the construction described in Theorem \ref{thm:linear} converges to a $t$-threshold function. \\

\begin{proof} [of Theorem \ref{thm:linear}.]  Let $f_R$ be the polynomial that describes the iterative construction in which $T_1$ and $T_2$ are selected with probability $t$ and $1-t$ respectively. Since, $f_{T_1}(p)=2p^2-p^3$ and $f_{T_2}(p)=p+p^2-p^3$, $$f_R(p)= t f_{T_1}(p) + (1-t) f_{T_2}(p) = (1-t)p + (1+t)p^2 - p^3.$$ Since $f_R(p)-p=p(1-p)(p-t)$, the fixed points of $f_R$ are $0, t, $ and $1$. We claim that $f_R$ exhibits linear convergence to a $t$-threshold function. 

Let $p$ be the probability that an input item fires. It suffices to consider the case when $p\leq t-1/n$. By Corollary \ref{cor:complementary}, convergence to 1 for $p \geq t+\frac{1}{n}$ follows from the complementary construction.

First we show that the probability an item at level $\Omega(\log n)$ fires is less than $\frac{t}{2}$. By definition $p-f(p)=p(1-p)(t-p)$. Observe that for $t/2<p\leq t- 1/n$ \begin{align*} 
\frac{t-f(p)}{t-p}= 1+\frac{p-f(p)}{t-p}
= 1+p(1-t)
\geq 1+\frac{t(1-t)}{2}.\end{align*}
It follows that for all $\ell$ either $f^{(\ell)}(p)< \frac{t}{2}$ or $$t-f^{(\ell)}(p) \geq \left(1+\frac{t(1-t)}{2}\right)^\ell (t-p)\geq \left(1+\frac{t(1-t)}{2}\right)^\ell \frac{1}{n}.$$
For $\ell= \log_{1+\frac{t(1-t)}{2}}\frac{tn}{2},$ $f^\ell(p)< \frac{t}{2}$. 

Next, we show that at $\Omega(k)$ additional levels, the probability an items fires is less than $2^{-k}$. For $p< \frac{t}{2}$, $$f(p)=p(1-p)(p-t)+p=p(1-(1-p)(t-p))
\leq p\left(1-\left(1-\frac{t}{2}\right)\frac{t}{2}\right).$$ 
It follows $$f^{(l)}(p)<\left(1-\left(1-\frac{t}{2}\right)\frac{t}{2}\right)^\ell p
<\left(1-\left(1-\frac{t}{2}\right)\frac{t}{2}\right)^\ell\frac{t}{2}.$$

Thus, for $l=\log_{\left(1-\left(1-\frac{t}{2}\right)\frac{t}{2}\right)}
 \frac{1}{t2^{k-1}}$, $f^\ell(p)< 2^{-k}$. We have shown that when the input items fire with probability $p\leq t-1/n$, items level $\Omega(k + \log n)$ will evaluate to $1$ with probability less than $2^{-k}$. \end{proof}



\subsection{Quadratic convergence from iterative trees with small building blocks}

In this section we show that using trees with four or five leaves as building blocks, we can construct an iterative tree that converges quadratically to a $t$-threshold function for restricted values of $t$. We begin with a lemma that provides sufficient conditions for quadratic convergence.

\begin{lemma} Let $f$ be a function corresponding to an iterative construction on $n$ inputs that satisfies the following conditions:\begin{itemize}
\item On the interval $[0,1]$, $f$ has precisely three fixed points: $0,t,$ and $1$.
\item (Linear Divergence) There exists constants $u,v$ satisfying $0<u<t$ and $t<v<1$ and constants $c_1, c_2>1$ such that 
\begin{enumerate}
\item $t-f(p)\geq c_1(t-p)$  for $p \in [u,t-\frac{1}{n}]$, and
\item $f(p)-t\geq c_2(p-t)$  for $p \in [t+\frac{1}{n},v]$.
\end{enumerate}
\item (Quadratic Convergence) For the constants $u,v$ as above, there exists constants $c_3, c_4$ such that $c_3 u<1$ and $c_4 (1-v)<1$ and 
\begin{enumerate}
\item $f(p)< c_3p^2$  for $p \in (0,u)$, and
\item $1-f(p)<c_4(1-p)^2$  for $p \in (v,1)$.
\end{enumerate}
\end{itemize} 
Then $f$ exhibits quadratic convergence to a $t$-threshold function, meaning items at level $\Omega(\log n+\log k)$ of the corresponding infinite width iterative construction compute a $t$-threshold function with probability at least $1-2^{-k}$. \label{conditions} \end{lemma}

\begin{proof} Let $p$ be the probability an input item fires. First we consider the case when $p\leq t-\frac{1}{n}$. By the linear divergence assumption, $t-f(p)\geq c_1(t-p)$  for $p \in [u,t-\frac{1}{n}]$. It follows that $f^{(\ell)}(p)< u$ or  $$t-f^{(\ell)}(p)\geq c_1^\ell (t-p)\geq c_1^\ell(1/n).$$
Thus for $\ell= \log_{c_1}n(t-u)$, $f^{(\ell)}(p)\leq u$. Therefore, level $\Omega(\log n)$ items fire with probability at most $u$. Next we show that given a level in which items fire with probability at most $u$, the items at $\Omega( \log k)$ levels higher in the iterative tree fire with probability at most $2^{-k}$. Let $p'$ be the probability an item fires at the first level for which the probability an item fires is below $u$. By the quadratic convergence assumption, $f(p')< c_3(p')^2$  for $p' \in (0,u)$. It follows that for $\ell> \log_2(\log_{c_3}(1/2))+\log_2(k)-\log_2(1+\log_{c_3}(t-1/n))+1$, $f^{(\ell)}(p')\leq c^{2^\ell-1}(p')^{2^\ell}\leq c^{2^\ell-1}u^{2^\ell}<2^{-k}$. We have shown that in expectation items at $\Omega( \log n+\log k)$ fire with probability less than $2^{-k}$ when $p\leq t-\frac{1}{n}$.
A similar argument applies for $p\geq t+\frac{1}{n}$.

 \end{proof}


\begin{remark} Let $f$ be a function corresponding to an iterative construction with fixed point $t$. Then there exists $u$ and $v$ for which the quadratic convergence condition of Lemma \ref{conditions} holds if and only if $f'(0)=0$ and $f'(1)=0$. \label{no linear term}\end{remark}

\begin{proof}  Quadratic convergence to $0$ is observed if and only if there exists some positive constant $u$ sufficiently close to $0$ for which all $x<u$, $f(x)= O(x^2)$. Writing $f(x)$ according to its Taylor series expansion about $0$
 implies that such behavior occurs if and only if $f'(0)=0$. Similarly, the observing the Taylor series expansion about $1$ allows us to conclude that quadratic convergence to $1$ is observed if and only if $f'(1)=0$. 
\end{proof} 

Next, we prove that the construction given in Theorem \ref{thm:quadratic four}A converges quadratically to a $t$- threshold function. \\

\begin{proof} [of Theorem \ref{thm:quadratic four}A.]
 Since  $2- \phi\leq t\leq \phi-1$, $0 \leq \alpha(t)\leq 1$ and the probability distribution $R$ is well-defined. By construction, $f_{F_1}(p)=4p^2-4p^3+p^4$ and $f_{F_2}(p)=2p^2-p^4 $, so $$f_R(p)= \frac{1 + t - 3 t^2 }{t(1 - t)} p^2+ \frac{-2 + 2t + 2t^2}{t(1 - t)}p^3+ \frac{1 - 2t}{t(1 - t)} p^4.$$
We apply Lemma \ref{conditions}. First note that $0, t,$ and $1$ are fixed points. Let $p$ be the fraction of input items firing. It suffices to show convergence to 0 when $p\leq t- \frac{1}{n}$. By Corollary \ref{cor:complementary}, convergence to 1 for $p\geq t+ \frac{1}{n}$ follows from the complementary construction. 

First we show linear divergence from $t$. Let $g(p)=\frac{f(p)-p}{p(1-p)(p-t)}= \frac{1}{t}+\frac{p(2t-1)}{(1-t)t}$. We claim $g(p)\geq 1$ for $0\leq p\leq 1$.  If $t\geq \frac{1}{2}$, then $g(p)\geq \frac{1}{t}>1$. If $t < \frac{1}{2}$, then $g(p)\geq \frac{1}{t}+\frac{2t-1}{(1-t)t}=\frac{1}{1-t} \geq 1$. 
Observe that for any constant $0<u<t$ and $u<p\leq t-1/n$  \begin{align*}t-f(p)=t-\left(p+p(1-p)(p-t)g(p)\right)
\geq (t-p)\left(1+p(1-p)\right)
\geq (t-p)\left(1+u(1-t)\right).\end{align*}
Thus $c_1=1+u(1-t)$ satisfies the first linear divergence condition.  

Next, we show quadratic convergence. Let $u=1/5$. Observe that $$f(p)\leq 4p^2-4p^3+p^4<4p^2.$$ Since $(1/5)4<1$, taking $c_3=4$ satisfies the first condition of quadratic convergence. 
Thus, we may apply Lemma \ref{conditions} to conclude that items at level $\Omega(\log n+\log k)$ in the limit of the iterative construction compute a $t$-threshold function with probability at least $1-2^{-k}$.

It remains to show that no construction using trees with four leaves will yield quadratic convergence to a $t$-threshold function for $t$ outside the range $2- \phi\leq t\leq \phi-1$. A $t$-threshold function with quadratic convergence must satisfy the following five constraints: (i) $f(0)=0$, (ii) $f(1)=1$, (iii) $f(t)=t$, (iv) $f'(0)=0$, (v) $f'(1)=0$. Solving these equations gives the function 
$$f(p) = \frac{1 + t - 3 t^2 }{t(1 - t)} p^2+ \frac{-2 + 2t + 2t^2}{t(1 - t)}p^3+ \frac{1 - 2t}{t(1 - t)} p^4.$$
Suppose that $f$ can realized by a convex combination of degree four polynomials. Then the leading coefficient of $f$ must be between $-1$ and $1$ since all achievable polynomials have leading coefficient $-1$ or $1$. Thus, $ 0 \leq \frac{1 - 2t}{t(1 - t)} \leq 1$, which implies that $2- \phi\leq t\leq \phi-1$. \end{proof}

\begin{proof} [of Theorem \ref{thm:quadratic four}B.] By construction $f_{V_1}(p)=p^2+p^3-p^5$ and $f_{V_2}(p)=6p^2-9p^3+5p^4-p^5$, so $$f_R(p)= \frac{1 + t - 2 t^2 - t^3}{t(1 - t)} p^2+ \frac{-2 + t + t^2 + 2 t^3}{t(1 - t)}p^3+ \frac{1 - t^2 - t^3}{t(1 - t)} p^4 -p^5.$$ 
We apply Lemma \ref{conditions}. 
 First note that $0, t,$ and $1$ are fixed points. Let $p$ be the fraction of input items firing. It show convergence to 0 when $p\leq t- \frac{1}{n}$. By Corollary \ref{cor:complementary}, convergence to 1 for $p\geq t+ \frac{1}{n}$ follows from the complementary construction. 

First we show linear divergence from $t$.  Let $g(p)=\frac{f(p)-p}{p(1-p)(p-t)}=\frac{1}{t}-\frac{t^2+t-1}{t(t-1)}p+p^2= \frac{1}{t}+p(-\frac{t^2+t-1}{t(t-1)}+p)$. We claim that $g(p)\geq1$ for $0\leq p\leq 1$.  If $-\frac{t^2+t-1}{t(t-1)}+p\geq 0$, then $g(p)\geq \frac{1}{t}\geq 1$. If $-\frac{t^2+t-1}{t(t-1)}+p <0$, then $g(p)\geq \frac{1}{t}-\frac{t^2+t-1}{t(t-1)}+1=\frac{1}{1-t}\geq 1$.
Thus as in the proof of part A, $c_1=1+u(1-t)$ satisfies the first linear divergence condition. 
Next we show quadratic convergence. Let $u=1/7$. Note that $$f(p)\leq p^2(6-9p+5p^2-p^3)<6p^2.$$ Since $6(1/7)<1$, taking $c_3=6$ satisfies the first condition of quadratic convergence. 
Thus, we may apply Lemma \ref{conditions} to conclude that in expectation items at level $\Omega(\log n+\log k)$ of the iterative construction compute a $t$-threshold function with probability at least $1-2^{-k}$.

It remains to show that no construction using trees with five leaves will yield quadratic convergence to a $t$-threshold function for $t$ outside the range $0.26 \lesssim t \lesssim 0.74$. A $t$-threshold function with quadratic convergence must satisfy the following five constraints: (i) $f(0)=0$, (ii) $f(1)=1$, (iii) $f(t)=t$, (iv) $f'(0)=0$, (v) $f'(1)=0$. Such a function will have the form:
\begin{align*} z_{d,t}(p)&=\frac{1 + t - (3 +d) t^2 + 
  d t^3}{(1 - t) t} p^2+ \frac{-2 + (2 +d)t +  (2+d) t^2  - 
   2 d t^3}{(1 - t) t} p^3 \\
& \quad +\frac{
  1 - (2+2d) t  + d t^2 + d t^3}{(1 - t) t}p^4 +dp^5.\end{align*}
 Since each achievable polynomial has leading coefficient $-1$ or $1$, if $z_{d,t}(p)$ can written as a convex combination of achievable polynomials of degree five, then $$z_{d,t}(p)= \beta z_{-1,t}(p)+ (d+\beta) z_{1,t}(p),$$ where $z_{-1,t}(p)$ and $z_{1,t}(p)$ are convex combinations of achievable polynomials of degree five with leading coefficient $-1$ and $1$ respectively and $0\leq \beta\leq 1$. Thus, it suffices to determine the values of $t$ for which $z_{-1,t}(p)$ is achievable through convex combinations and the values of $t$ for which $z_{1,t}(p)$ is achievable through convex combinations.
  
\underline{Claim:} Let $\alpha(t)= \frac{-1 + 5 t - 4 t^2 + t^3}{5 t(t-1)}$. If the function $z_{-1,t}(p)$ is achievable through convex combinations then $0\leq \alpha(t)\leq 1$, meaning $0.26 \lesssim t \lesssim 0.74$.

\smallskip

Notice that achievable polynomials of degree five with leading coefficient $-1$ have coefficient $a_3\geq -9$ (see Table \ref{the table}). It follows that $$\frac{-2 + t + t^2 + 2 t^3}{t(1 - t)}\geq -9 \quad \text{ and} \quad
2( -1 + 5 t - 4 t^2 + t^3) \geq 0,$$ so $\alpha(t)\geq 0$. Next, note that the coefficient $a_4$ of $z_{-1,t}(p)$  must be non-negative (see Table \ref{the table}). It follows that $$ \frac{1 - t^2 - t^3}{t(1 - t)} \geq 0 \quad \text{and} \quad -1+5t-4t^2+t^3\leq 5t-5t^2,$$ so $\alpha(t)\leq 1$.

\underline{Claim:} Let $\gamma(t)=1-2t-t^2+t^3$ and $\beta(t)=1-3t^2+t^3$. If the function $z_{1,t}(p)$ is achievable through convex combinations, then $\gamma(t)\leq0$ and $\beta(t)\geq 0$, meaning $.445 \lesssim t \lesssim .653$ 

Assume $z_{-1,t}(p)$ is achievable through convex combinations. Notice that for degree five achievable polynomials with $a_1=0$ and $a_5=1$, $-4 \leq a_4 \leq -2$. It follows that 
$$-4 \leq \frac{ 1 - 4 t  +  t^2 +  t^3}{(1 - t) t}\leq -2,$$ so $\gamma(t)\leq 0$ and $\beta(t) \geq 0$. 
 
 \smallskip
 
Now consider $t\lesssim 0.26$ or $ t \gtrsim 0.74$. By the above claims, $z_{-1,t}(p)$ and $z_{1,t}(p)$ are not achievable through convex combinations. It follows that $z_{d,t}(p)$ is not achievable through convex combinations, 
meaning no construction on trees with five leaves that converges quadratically to a $t$-threshold function for $t\lesssim 0.26$ or $ t \gtrsim 0.74$.\end{proof}


Using a similar technique as in the proof above, it is possible to show that the analogous constructions on six and seven leaves yield iterative constructions that converge quadratically to  threshold functions for thresholds in the ranges $0.15 \lesssim t \lesssim 0.85$ and $0.11 \lesssim t \lesssim 0.89$ respectively.  However, it is not possible to generalize such a construction beyond this point. Instead, we observe the emergence of a staircase functions, which will be discussed in Example \ref{soft threshold}.

\subsection{ Quadratic convergence for arbitrary thresholds.} 
In this section we show that as $t$ approaches 0 or 1, increasingly large building blocks trees are needed to construct an iterative tree that converges quadratically to a $t$- threshold function. Further, we give a construction that exhibits quadratic convergence for arbitrary thresholds near $0$ and $1$. We begin by proving Theorem \ref{thm:degreenew}, which can be restated as follows: Let $f$ be an achievable polynomial with fixed points $0$, $t$, and $1$ that exhibits quadratic convergence to a $t$-threshold function. Then, $f$ has degree at least $\frac{1}{\sqrt{2s}}$ where $s= \min\{t, 1-t\}$.

\begin{proof} [of Theorem \ref{thm:degreenew}.]  Let $f$ be an achievable polynomial with fixed points $0$, $t$, and $1$ that exhibits quadratic convergence. Then for $\ve$ sufficiently small, $f(\ve)= O(\ve^2)$, which implies $a_1=0$. For $x < \frac{1}{2d}$, we have 
\begin{align*}
f(x)&= a_2 x^2 +a_3x^2 + \dots + a_d x^d
\leq d^2x^2+d^3x^3 +\dots d^d x^d
< d^2x^2 \left( \frac{1}{1-dx}\right)
< 2d^2x^2. \end{align*}
Since $t$ is a fixed point of $f$, $f(t)=t$. Thus, $t< 2d^2t^2$. It follows that $d> \frac{1}{\sqrt{2t}}$. By Lemma \ref{complement}, if there exists an achievable polynomial with fixed point $t$, then there also exists a complementary achievable polynomial with fixed point $1-t$. Thus, $d> \frac{1}{\sqrt{2(1-t)}}$. \end{proof}

We now prove that a nearly matching iterative construction exists. 
To achieve quadratic convergence to thresholds near 0 or 1, we average trees of the form $A_k$ and $A_{k+1}$ or $B_k$ and $B_{k+1}$ respectively. \\

\begin{proof} [of Theorem \ref{thm:k leaves}.]  By Corollary \ref{cor:complementary}, it suffices to prove the theorem for $1-\phi \leq  t< 1$. The complement of a construction that achieves quadratic convergence to a $t$-threshold function yields quadratic convergence for to a $(1-t)$-threshold function. By Lemma \ref{k,k+1}, there exists $k$ and $\alpha$ such that $f= \alpha f_{B_k} + (1-\alpha) f_{B_{k+1}}$has fixed point $t$. Moreover, $\frac{t-f(p)}{t-p} \geq  \left(1+ \frac{p(1-p)}{t}\right)$. 

We apply Lemma \ref{conditions} to prove that $f$ converges to a $t$-threshold function. Let $p$ be the probability an input item is on. First suppose that $p\leq t-\frac{1}{n}$. We show linear divergence away from $t$. For any constant $0<u<t$, and $u\leq p\leq t-\frac{1}{n}$ by Lemma \ref{k,k+1} we have \begin{align*} t-f(p)\geq (t-p)  \left(1+ \frac{p(1-p)}{t}\right) 
\geq (t-p) \left(1+ \frac{u(1-t)}{t}\right). \end{align*}
Thus, $c_1=1+ \frac{u(1-t)}{t}$ is a valid choice for $c_1$ in Lemma \ref{conditions}.

Next, we claim that $u=1-\frac{1}{k-1}$ is a valid starting point for quadratic convergence towards 0.  We write $f(p)=p^2(\alpha d_k(p)+ (1-\alpha) d_{k+1}(p))$ where $d_k(p)= 2p^{k-2}-p^{2k-2}$. Let $d(p)=\alpha d_k(p)+ (1-\alpha) d_{k+1}(p)$. Note that $d(p)$ is increasing on the interval $(0, u)$ since each $d_k$ increases on this interval. For $p< u$, $$\frac{2k-4}{2k-2}=u>u^k>p^k.$$ It follows that 
$d_k'(p)=p^{k-3}((2k-4)-(2k-2)p^k)>0$. Thus, $d_k$ is increasing on the interval $(0,u)$. Thus, $c_3= d(u)$ is a valid choice for $c_3$ in Lemma \ref{conditions}.

It remains to show that for $p\geq t+\frac{1}{n}$ we observe linear divergence from $t$ then quadratic convergence to 1.
We show linear divergence away from $t$. For any constant $t<v<1$, and $t+\frac{1}{n} \leq p\leq 1$ by Lemma \ref{k,k+1} we have \begin{align*} f(p)-t\geq (p-t)  \left(1+ \frac{p(1-p)}{t}\right) 
\geq (p-t) \left(1+ \frac{t(1-v)}{t}\right). \end{align*}
Thus, $c_2=1+ \frac{t(1-v)}{t}$ is a valid choice for $c_2$ in Lemma \ref{conditions}.   

We claim that $v> 1-\frac{1}{8(k+1)^2}$ is a valid starting point for quadratic convergence to $1$. By Corollary \ref{cor:complementary}, $f_{A_k}(1-p)=1-f_{B_k}(p)$. It follows
$$1-f(p)= \alpha- \alpha f_{B_k}(p) + (1-\alpha) -(1-\alpha) f_{B_{k+1}}(p)=  \alpha f_{A_k}(1-p) + (1-\alpha) f_{A_{k+1}}(1-p).$$
Recall from the proof of Theorem \ref{thm:degreenew},  $f(x)< 2dx^2$ where $d$ is the degree of $x$. Therefore, 
$$f_{A_{k}}(1-p)< 8k^2(1-p)^2<8(k+1)^2(1-p)^2 \text{ and } f_{A_{k+1}}(1-p)<8(k+1)^2(1-p)^2.$$ Since $(1-v)8(k+1)^2<1$, $c_4=8(k+1)^2$ is a valid choice for $c_4$ in Lemma \ref{conditions}. \end{proof}

\section{Convergence of iterative constructions to staircase functions}\label{steps}

In this section, we explore the possible functions that can be achieved by sampling from the high-level items of iterative constructions with no restrictions on the size of building block trees. 
\begin{defn} The function $f$ is a staircase function if $f(x)=p_i$ on the interval $(a_i, a_{i+1})$  where $a_0=0<a_1<a_2<\dots <a_k< a_{k+1}=1$, $0=p_0<p_1< p_2< \dots< p_{k-1}<p_k=1$. \end{defn}
We will show in Theorem \ref{thm:staircase} that any staircase function in which each step intersects the line $y=x$, or more precisely $a_i\leq p_i\leq a_{i+1}$ for all $0 \leq i \leq k$, can be approximated by a high-level of an iterative tree. 

This result may seem to suggest that the high-level items of these constructions are pseudorandom.  However, this is not the case. The output of a high-level item is highly influenced by which of tree it is the root of. The distribution of items firing at a given high-level behaves according to the staircase function, rather than each item individually behaving according to the staircase function. 

We begin by giving a couple of examples of staircase functions arising from probability distributions of small trees. The first example is a one step staircase. 

\begin{example} Consider the iterative construction that selects a tree that computes $A \wedge B \wedge C$ with probability $\alpha$ and a tree that computes $A \vee B \vee C$ with probability $1-\alpha$. Let $f(p)= \alpha (1-(1-p)^3) + (1-\alpha) p^3$ be the corresponding polynomial. For $\alpha \in (1/3, 2/3)$, $f$ converges to an $(3\alpha-1)$ one step staircase function, i.e. takes value $3\alpha -1$ on $(0,1)$. For $\alpha=1/2$ and inputs in which the fraction of inputs firing is not 0 or 1, approximately half the high-level items will fire. 
 Figure \ref{onestep} illustrates $f$ for $\alpha=1/2$. \label{one step}  \end{example}

 \begin{figure}[h]
 \centering
\includegraphics[width=.2\textwidth]{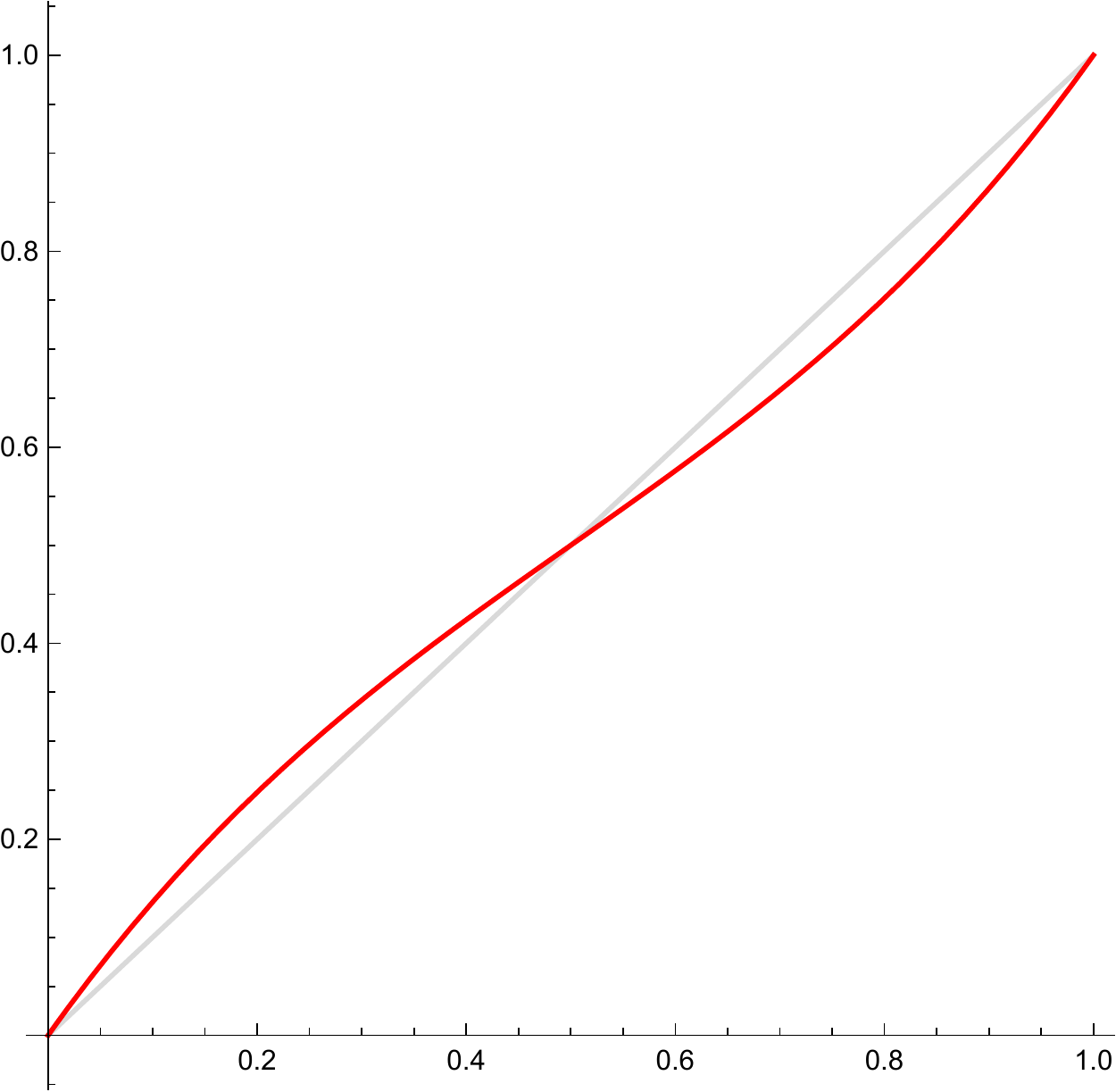}
\quad
\includegraphics[width=.2\textwidth]{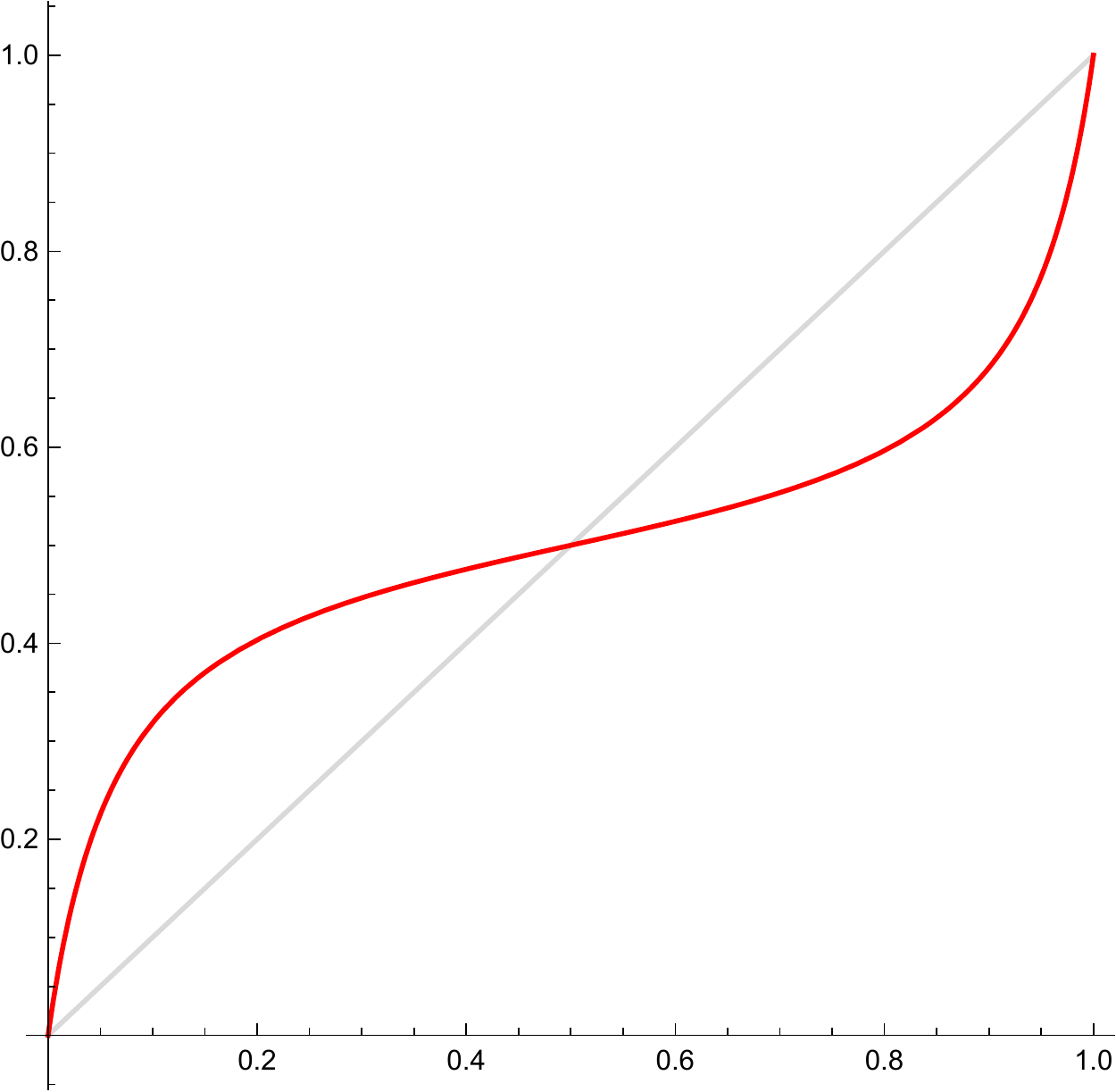}
\quad
\includegraphics[width=.2\textwidth]{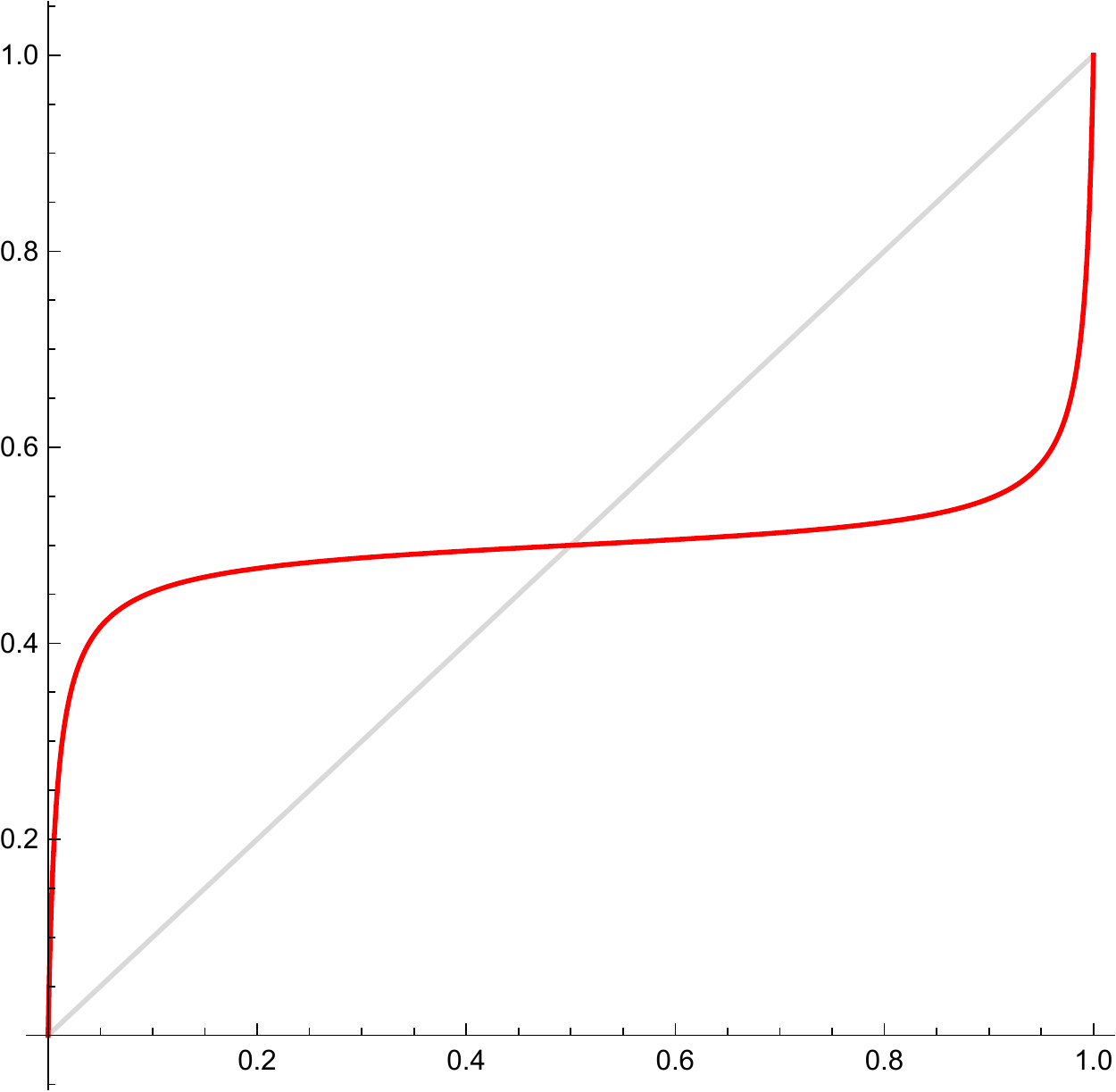}
\quad
\includegraphics[width=.2\textwidth]{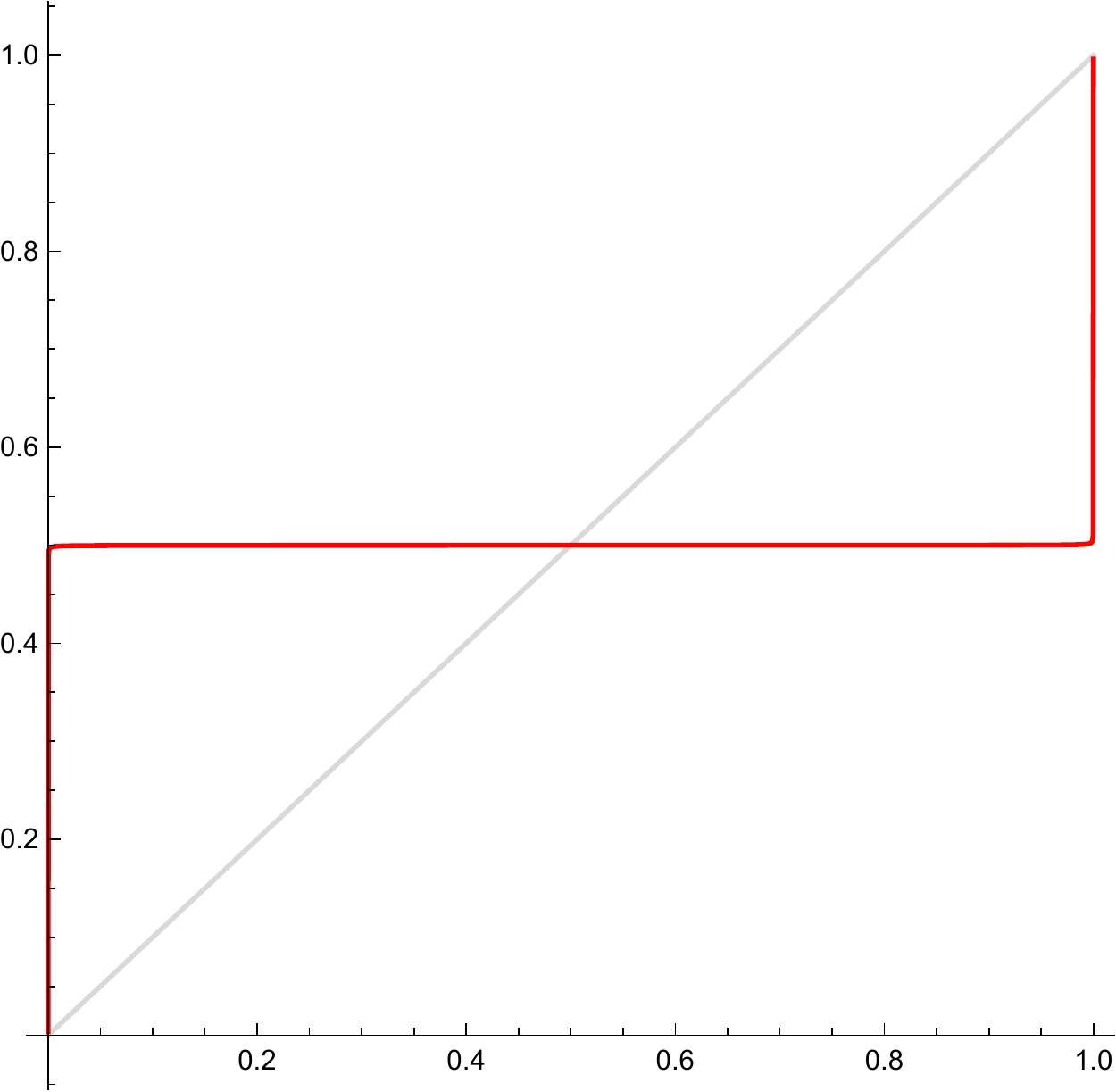}
  \caption{An iterative construction for the above function converges a 1/2 one step staircase function. Left to right:  $f(p), f^{(5)}(p),f^{(10)}(p),f^{(30)}(p)$.}
  \label{onestep}
  \end{figure}

\begin{lemma}  Consider the iterative construction that selects a tree that computes $A \wedge B \wedge C$ with probability $\alpha$ and a tree that computes $A \vee B \vee C$ with probability $1-\alpha$ where $\alpha \in (1/3, 2/3)$. Let ${f(p)= \alpha (1-(1-p)^3) + (1-\alpha) p^3}$ be the corresponding polynomial. Then $3\alpha -1$ is the only fixed point of $f$ on $(0,1)$ and it is an attractive fixed point. \end{lemma}

\begin{proof} The only fixed points of $f(p)$ are the roots of $f(p)-p$, which are $0,1, 3 \alpha -1$. To prove that $3\alpha-1$ is an attractive fixed point, we show that $f'(3\alpha-1)< 1$. We compute $f'(p)= 3 \alpha (1-p)^2 +3(1-\alpha)p^2$. For $\alpha \in (1/3,2/3)$, we have $$f'(3 \alpha-1) =3-9\alpha +9 \alpha^2<1.$$ \end{proof} 

We remark that roughly $\log(1/\delta) + \log(1/\ve)$ levels are needed (and suffice) to get within $\delta$ of $\alpha=1/2$ where $\ve$ is the distance of the input from $1/2$. We omit the proof, which is similar to other proofs of finite bounds in this paper.


In the above construction, a randomly sampled high-level item fires with probability approximately half. However, whether a fixed high-level item fires is not pseudorandom.  Given that  half of the items at level $k-1$ fire, an item at level $k$ that is the root of an $A \wedge B \wedge C$ building block tree fires with probability $1/8$. An item at level $k$ that is the root of a $A \vee B \vee C$ building block tree fires with probability $7/8$. 

Another simple example of a staircase function is the generalization of the construction given in Theorem \ref{thm:quadratic four}. 

\begin{example} For $k\geq 4$, the function  $h=(f_{A_k}+f_{B_k})/2$ converges to a three step staircase function. The function $h$ has a non-attractive fixed point $s \in (0,1/2)$, a non-attractive fixed point $t \in (1/2,1)$, and an attractive fixed point at $1/2$. Figure \ref{surprise} illustrates $h_6(p)$. Therefore, for inputs in the interval $(s,t)$, approximately half the high-level items return $0$ and half return $1$. For inputs in the intervals $[0,s]$ and $[t,1]$, high-level items return $0$ and $1$ respectively with high probability. \label{soft threshold} \end{example}

 \begin{figure}[h]
 \centering
\includegraphics[width=.2\textwidth]{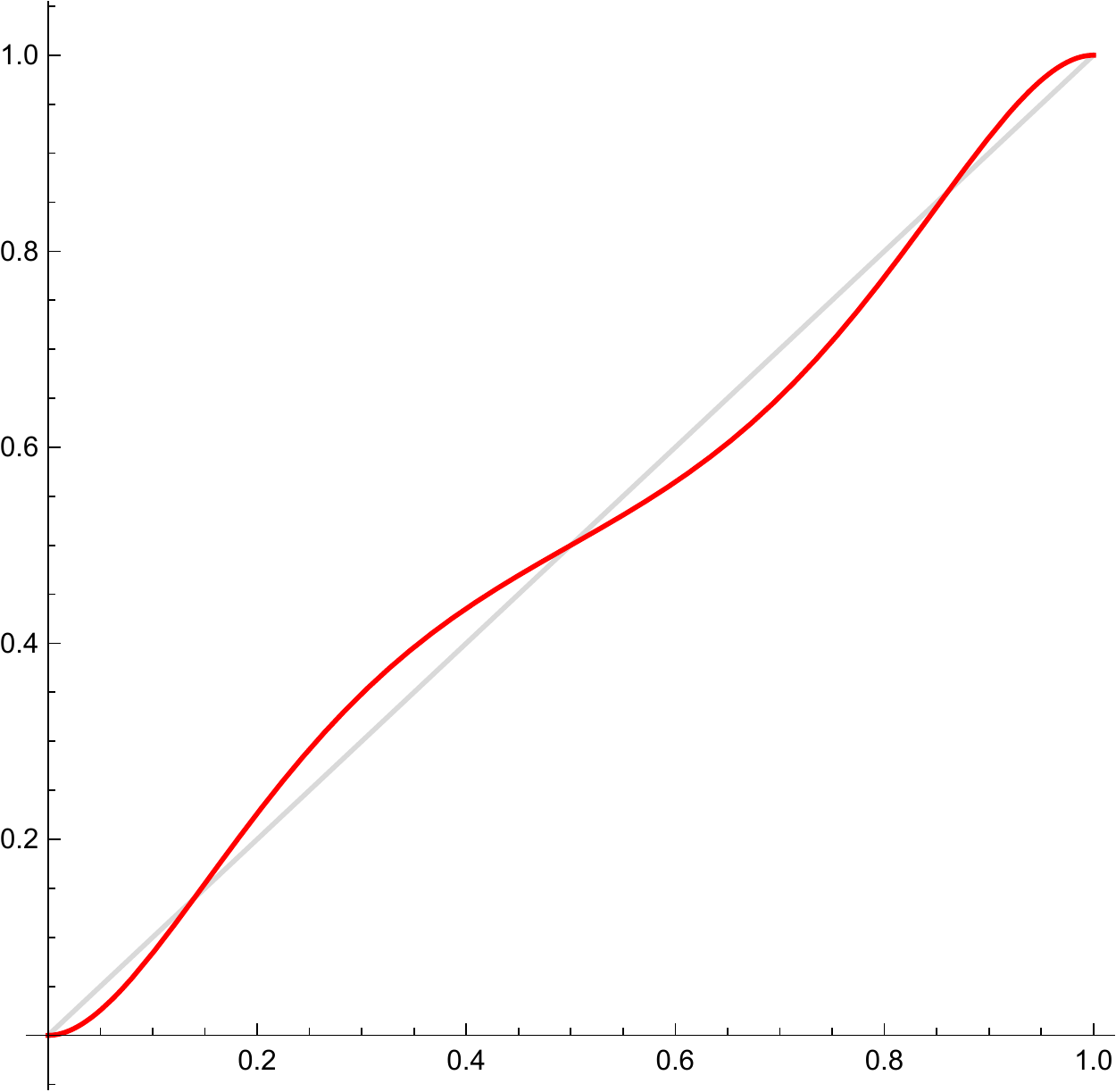}
\quad
\includegraphics[width=.2\textwidth]{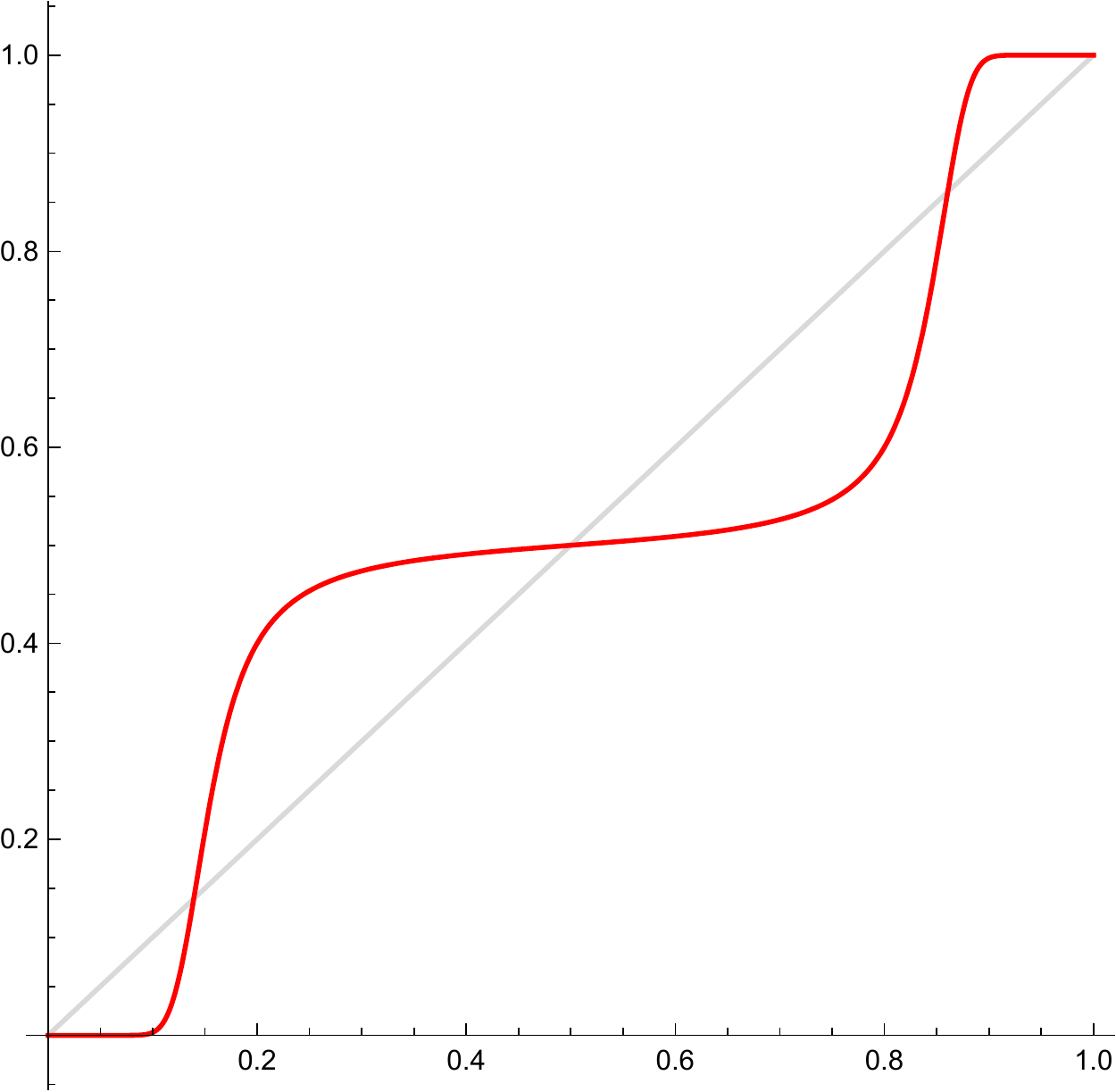}
\quad
\includegraphics[width=.2\textwidth]{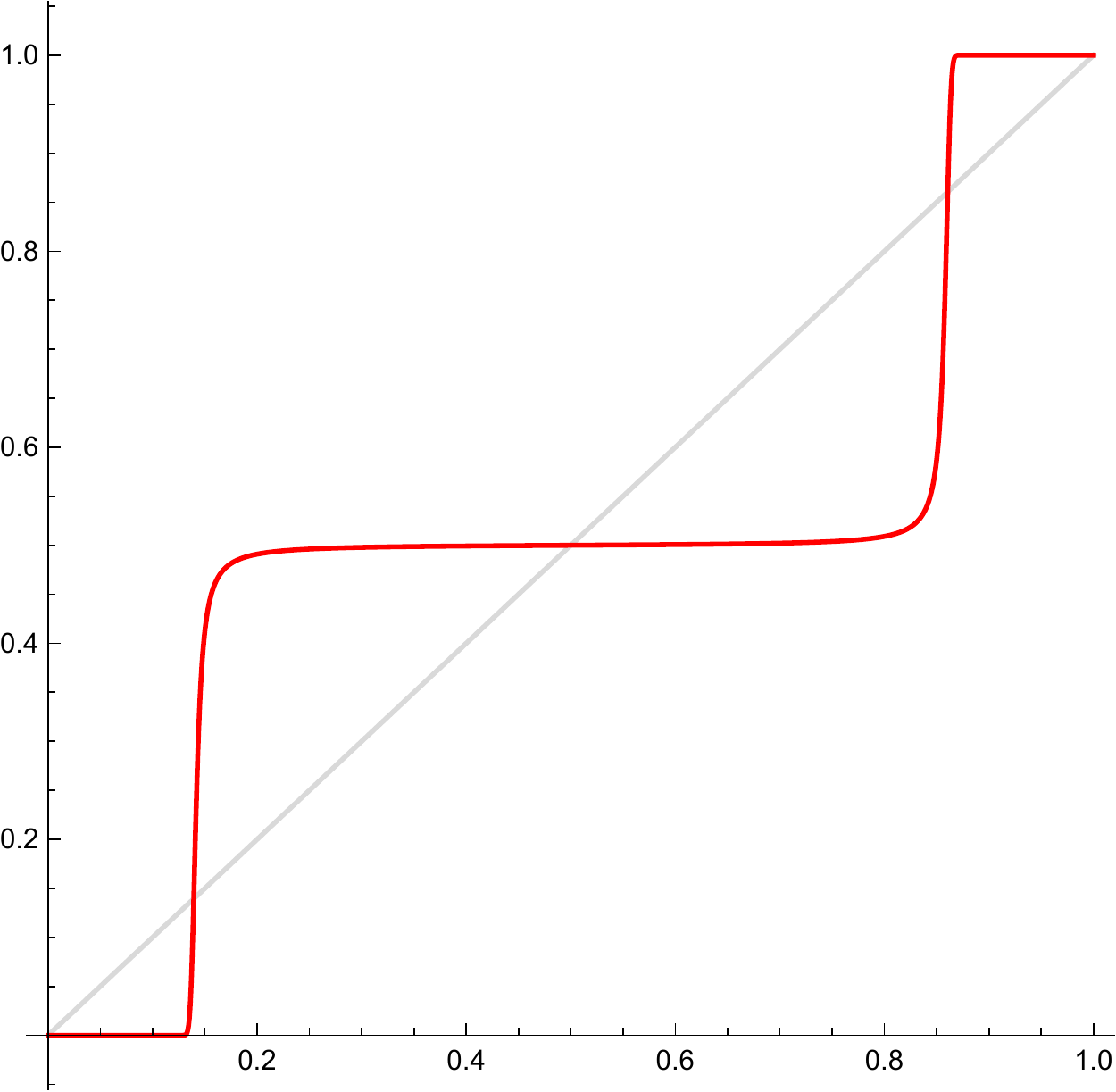}
\quad
\includegraphics[width=.2\textwidth]{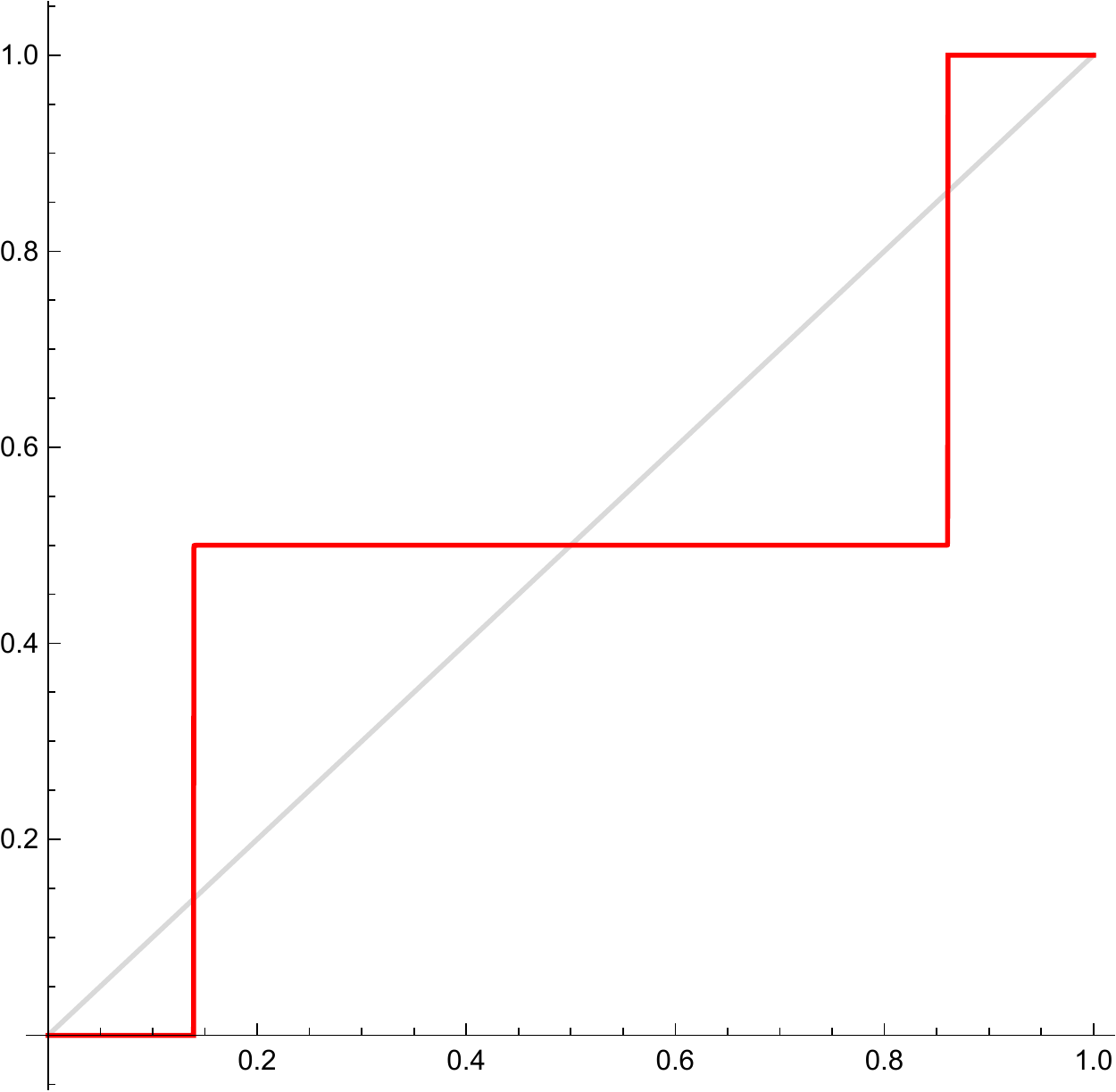}
  \caption{For $k \geq 4$, the function $(f_{A_k}+f_{B_k})/2$ converges to a three step staircase function. Left to right:  $f(p), f^{(5)}(p),f^{(10)}(p),f^{(30)}(p)$.}
  \label{surprise}
  \end{figure}

\begin{lemma} Let $k\geq 4$. Consider an iterative construction in which $A_k$ and $B_k$ are each selected with probability $1/2$. Let $h=(f_{A_k}+f_{B_k})/2$ be the corresponding polynomial. Then $1/2$ is an attractive fixed point of $h$.  \label{lem:soft threshold}
 \end{lemma} 

\begin{proof} Let $h=(f_{A_k}+f_{B_k})/2$. It suffices to show  $h'(1/2)<1$ for $k \geq 4$. We compute $h'(p)=k\left(p^{k-1}+(1-p)^{k-1}\right) -k\left(p^{2k-1}+(1-p)^{2k-1}\right)$. It follows  for $k \geq 4$ that $$ h'(1/2)=  k\left(2^{-(k-2)} -k2^{-(2k-2)}\right)=\frac{k}{2^{k-2}}\left(1-2^{-k}\right) <1.$$ \end{proof}

Finally, we show that it is possible to approximate any staircase function.  The proof will rely on Lemma \ref{straight} and Lemma \ref{dense}, which describe the scope of achievable polynomials for a single  tree. 

\begin{lemma} Let $t$ be a fixed point of some achievable polynomial $f$ with attractive fixed points $0$ and $1$, $0< t<1$. Then, for any $\delta,\ve>0$ there exists an achievable polynomial $g$ such that $f(x)< \delta$ for $x \in [0, t- \ve]$ and $f(x)>1- \delta$ for $x \in [t+ \ve,1].$ \label{straight}
\end{lemma}

\begin{proof} First we claim that if $f$ is achievable, then  $f^{(k)}$ is also achievable. Let $f$ be the polynomial corresponding to a tree $T$. Define $T^1$ to be the tree $T$ modified so that each of the leaves is replaced with a copy of $T$. Let $T^k$ be the tree $T^{k-1}$ modified so that each of the leaves is replaced with a copy of $T$. Note that $f^{(k)}$ is the polynomial for $T^k$. 

Let $f$ be the polynomial for which $t$ is a fixed point. It suffices to show that there exists some $k$ such that $f^{(k)}(x)< \delta$ for $x \in [0, t- \ve]$ and $f^{(k)}(x)>1- \delta$ for $x \in [t+ \ve,1].$ Since $t$ is the only fixed point and 0 is an attractive fixed point, for $p \in [\delta, t-\ve]$, $f(p)<p$. Let $m= \min_{ p \in [\delta, t-\ve]} p-f(p)$. Then for all $\ell$, $f^{(\ell)}(p) \leq t-\ve -\ell m$ or $f^{(\ell)}(p) < \delta$. Let $x \in [0, t-\ve]$. Then for $k = \lceil \frac{t-\ve}{m} \rceil$, $f^{(k)}(x)\leq f^{(k)}(t-\ve) \leq \delta$. A similar argument proves the case when $x \in [t+\ve, 1]$.  \end{proof}

The following lemma says that with a single tree, the set of achievable fixed points is dense in $(0,1)$. A weaker version of this lemma in which the building blocks of the constructions are arbitrary circuits rather than trees is given by \citep{Moore1956}. Our proof is inspired by their argument. 

\begin{lemma} The set of fixed points of achievable polynomials is dense in $[0,1]$.\label{dense} \end{lemma}
The proof will rely on the following claims. 

\begin{claim} Let $f,g$ be achievable polynomials such that $g\leq p$ and $f$ has fixed point $r$, $0<r<1$. Let $x= g^{-1}(r)$. Then for all sufficiently small $\ve >0$, there exists $k$ such that the fixed point of $f^{(k)}(g(p))$ is in $(x, x+\ve)$. \label{claim one} \end{claim}

\begin{proof} Note that for any $k$, $f^{(k)}(g(x))=r<x$. Therefore, to prove that there is a fixed point in $(x, x+ \ve)$, it suffices to show that $f^{(k)}(g(x + \ve)) > x + \ve$. Since $g $ is increasing, $g(x+ \ve)= r +\delta $ for some $ \delta >0$. By the same arguement in the second paragraph of the proof of Lemma \ref{straight}, there exists $k$ such that $f^{(k)}(r +\delta) > \alpha $ for any $\alpha$. Taking $\alpha= x +\ve$, we have $x+\ve < f^{(k)}(r+\delta) = f^{(k)}(g(x +\ve))$. \end{proof}

\begin{claim}\label{triple} Let $T_j$ be the tree that computes $(x_1 \vee x_2 \vee \dots \vee x_j) \wedge x_{j+1}$ and let $g(p,j)= p(1-( 1-p)^j)$ be the corresponding polynomial. Then \begin{enumerate} 
\item For all $\ve>0$ and fixed $p \in (0,1)$, there exists $j$ such that $p- g(p,j) < \ve$. 
\item For all $p> 2-\phi \approx .38$ and $j \geq 2$, $p- g(p,j)\geq g^{-1}(p,j) -p$. 
\item For $j \geq 2$, $h(p)=p-g(p,j)$ is decreasing on $(.38 ,1)$.  
\end{enumerate}
\end{claim}

\begin{proof} Statement (1) follows from the limit computation $$\lim_{j \to \infty } p- g(p,j)= \lim _{j \to \infty} p(1-p)^j=0.$$ For (2) we first observe that for $p \in [2- \phi, 1]$, $$g'(p,j)=1-(1-p)^{j}+pj(1-p)^{j-1}=1-(1-p)^{j-1}(1-p(1+j))>1$$ where $g'(p,j)$ denotes the derivative with respect to $p$. For $p \in [2- \phi, 1]$, $$\frac{p-g(p)}{g^{-1}(p)-p} > \min_{k   \in [2- \phi, 1]} g'(k)>1.$$
Finally, for (3), note that $h'(p)=1-g'(p,j)< 0$ for $p \in  [2- \phi, 1]$. 
\end{proof}

\begin{proof}[of Lemma \ref{dense}]. To prove that the set of achievable fixed points is dense in $(0,1)$, it suffices to show that for all $\ve >0$, there exists a set of achievable fixed points $S$ such that for all $x \in (1/2, 1)$, there exists $s \in S$ such that $|x-s|< \ve$. Corollary \ref{cor:complementary} will imply density in $(0,1/2)$. 

We construct $S$ as follows. Let $f_0= f_{B_2}$. Recall $f_0$ has fixed point $p_0=2-\phi\approx .38$. Let $g(p)= p(1-(1-p)^j) $ where $j$ is chosen according to Claim \ref{triple}.1 so that $p_0-g(p_0) < \ve/2$.  We define $p_i$ and $f_i$ inductively, where $f_i$ is an achievable polynomial with fixed point $p_i$. Let $x_{i+1}=g^{-1}(p_i)$. By Claim \ref{claim one}, there exists $k$ such that the fixed point of $f_i^{(k)}(g(x_{i+1}))$ is less than $ x_{i+1}+ \ve/2$. Define $f_{i+1}= f_i^{(k)}(g(p))$ and $p_{i+1}$ to be the corresponding fixed point. Apply Claim \ref{triple}, we observe $$p_{i+1}-p_i < x_{i+1} + \frac{\ve}{2} -p_i = g^{-1}(p_i)-p_i+ \frac{\ve}{2} \stackrel{\ref{triple}.2}{\leq }  p_i-g(p_i)+\frac{\ve}{2}\stackrel{\ref{triple}.3}{\leq } p_0-g(p_0)+\frac{\ve}{2}\stackrel{\ref{triple}.1}{\leq }  \ve.$$ It follows that $S=\{p_0, p_1, \dots \}$ satisfy the desired property that every point in $(1/2,1)$ is within $\ve$ of some $p_i \in S$.  
\end{proof}

\begin{theorem} For any $0=p_0< p_1< \dots < p_{k}=1$,  $0=a_0< a_1< \dots a_{k} <a_{k+1}=1$, $\delta, \ve >0$ such that $\ve < a_{i+1}-a_i$ for all $i$, there exists a probability distribution on a set of trees such that for $0 \leq i \leq k$ the corresponding polynomial $f$ has the property that $|f(x)-p_i|< \delta$ for all $x \in (a_i + \ve, a_{i+1}-\ve)$. 
\label{thm:staircase}
\end{theorem}

\begin{proof} We define a probability distribution of a set of trees, the existence of which are guaranteed by Lemma \ref{dense} and Lemma \ref{straight}. For $1\leq i \leq k$, let $t_i$ be an achievable fixed point in $(a_i-\ve/2, a_i+\ve/2)$. Let $f_i$ be an achievable function with fixed point $t_i$ the property that $ f_i(x) < \delta $ for $x \in [0, t_i-\ve/2]$ and $f_i(x) > 1-\delta$ for $x \in [t_i+\ve/2,1]$. It follows that $ f_i(x) < \delta $ for $x \in [0, a_i-\ve]$ and $f_i(x) > 1-\delta$ for $x \in [a_i+\ve,1]$. Let $\alpha_1=p_1$ and $\alpha_i=p_i-\sum_{j=1}^{i-1} \alpha_j$ for $2\leq i\leq k$. Let $f(x)=\sum_{j=1}^k \alpha_j f_j$. 

We show $f$ is the desired function by computing $f(x)$ for $x \in (a_i + \ve, a_{i+1}-\ve)$. Note that $f_j(x) < \delta$ for $j\leq i-1$ and $f_j(x) > 1-\delta$ for $j \geq i$. Observe $$p_i -\delta \leq p_i(1-\delta) \leq f(x)= \sum_{j=1}^{i} \alpha_j f_j +  \sum_{j=i+1}^{k} \alpha_j f_j  \leq p_i+ \delta (1-p_i) \leq p_i +\delta.$$ Thus, $|f(x)-p_i|< \delta$ for all $x \in (a_i + \ve, a_{i+1}-\ve)$. \end{proof}

Theorem \ref{thm:staircase} implies that any staircase function for which $a_i\leq p_i\leq a_{i+1}$ for all $0 \leq i \leq k$, can be approximated by the distribution of items at a high-level of an iterative tree. The condition that $a_i\leq p_i\leq a_{i+1}$ for all $0 \leq i \leq k$ guarantees the heights of each step are fixed points, and therefore the stairs are maintained when $f$ is iterated.

\section{Finite iterative constructions of threshold trees}\label{finite}

In the above section, we analyzed the behavior of iterative trees in the limit with respect to level width. We assumed that for any input the number of items turned on at level $l$ of the tree is equal to its expectation, $mf^{(l)}(p)$ where $m$ is the width of level $l$ and $p$ is the fraction of the inputs turned on. In a ``bottom up" construction in which the items of one level are fixed before the next level is built, we note that the chance that the number of items that fire at a given level deviates from expectation is non-trivial. In this section, we give a bound on the width of the levels required to achieve a desired degree of accuracy for finite realizations of iterative constructions.

\begin{remark} We can use a transition matrix to directly compute the probability that a high-level item of an iterative construction fires given the width of the levels. Let $f$ be the function corresponding to the construction, $p$ be the fraction of input items firing, and $m$ the width of the levels.  Define $s\in \R^{1\times (m+1)}$, $A\in \R^{(m+1) \times (m+1)}$, and $t\in \R^{1 \times (m+1)}$ 

$$s_i= {m \choose i} f(p)^i \left(1-f(p)\right)^{m-i}, \quad A_{i,j}= {m \choose j} f\left(\frac{i}{m}\right)^j \left(1-f\left(\frac{i}{m}\right)\right)^{m-j}, \quad t_i=i$$
for $i,j=0, 1, \dots m$. 
Then the probability that an item at level $L$ fires is $sA^{L-1}t^T$.\end{remark}

We will use the following concentration inequalities.

\begin{lemma} (Chernoff) Let $Y_1, Y_2, \dots Y_m$ be independent with $0\leq Y_i \leq 1$ and 
$Y=\sum_{i=1}^n Y_i$. Then, for any $\delta>0$,  $$\Pr(Y -E(Y) \geq \delta E(Y)) \leq \e\left(\frac{-\delta^2 E(Y)}{2+\delta}\right).$$ \label{concentration}
\end{lemma}

\begin{lemma} Let $X$ be a sum of $n$ binomial random variables with mean $\mu$.  Then, for $k \geq n\mu$, 

$$\Pr(X\geq k ) = \sum_{i=k}^n  {n \choose i} \mu^i (1-\mu)^{n-i}  < \e\left( - n H(\mu,k/n)\right)$$
where $H(p,q) = q \log (q/p) + (1-q)\log((1-q)/(1-p)).$ \label{binomial} \end{lemma}

For ease of notation,  all statements in this section about the probability of $X_{i+1}$ taking some values refers to the probability of  $X_{i+1}$ taking some values given $X_i$.

The following lemma describes linear divergence for finite width constructions. 
\begin{lemma}  Consider the construction of a $t$-threshold function in which each level  $\ell$ has $m_\ell$ items and the fraction of input items firing is at least $\ve$ below the threshold $t$. Let $d$ be the minimum value of $\frac{f(p)-p}{ p(1-p)(p-t)}$ on the interval $[0,1]$. Then, with probability at least $1-\gamma$, the fraction of inputs firing at level $\Omega(\frac{1}{\ve})$ will be less than any fixed constant $u$ when 
\[
m_\ell=\frac{8\ln(\frac{1}{u(1-t)\gamma})}{d^2u(1-t)^2 \left(1+\frac{c_1}{2}\right)^{\ell-1}\ve^2}
\] 
where $c_1$ is the linear divergence constant. 
\label{finite linear} \end{lemma}

\begin{proof} Let $X_i$ be the fraction of items firing at level $i$. Then  $E(X_i)= f(X_{i-1})$. 
In expectation, the sequence $X_1, X_2, X_3, \dots $ converges to $0$. We will show that with probability at least $1- \gamma$, the sequence obeys the \textit{half-progress} relation $X_{i+1}\leq \frac{X_i+ f(X_i)}{2}$ and therefore $X_L< u$ for $L=\Omega(\frac{1}{\ve})$. 

Write $f(p)-p= p(1-p)(p-t)g(p)$ where $g$ is a polynomial in $p$. Let $d$ be the minimum value obtained by $g$ on the interval $[0,1]$. First we compute probability that $X_{i+1} >\frac{X_i+ f(X_i)}{2}$ given $X_i$ by applying Lemma \ref{concentration}.  Observe \begin{align*} \Pr\left(X_{i+1}>\frac{X_i+ f(X_i)}{2}\right)&= \Pr\left(X_{i+1}-E(X_{i+1})>\frac{X_i- f(X_i)}{2}\right)\\
&\leq \e\left(\frac{-\left(\frac{X_i- f(X_i)}{2 f(X_i)}\right)^2 m f(X_i)}{2+\frac{X_i- f(X_i)}{2 f(X_i)} }\right)\\
&=\e\left(-\frac{(X_i(1-X_i)(X_i-t)g(X_i))^2m}{2(X_i+3(X_i+X_i(1-X_i)(X_i-t) g(X_i)))}\right)\\
&\leq \e\left(- \frac{X_i(1-X_i)^2(t-X_i)^2 d^2m}{8}\right)
\end{align*}
Let $\ve_i=t-X_i$ and $\alpha=\frac{u(1-t)^2d^2}{8}$. Then for $u \leq X_i \leq t- \ve$, $$\Pr\left(X_{i+1}>\frac{X_i+ f(X_i)}{2}\right)<\e\left(-\alpha m \ve_i^2\right).$$

Next we compute the probability that $i$ is the first value for which the half-progress relation is not satisfied given $X_i>u$. If the half-progress relation is satisfied meaning $X_{i+1}>\frac{X_i+ f(X_i)}{2}$, then $\ve_{i+1}\geq \ve_i \beta$ where $\beta= 1+\frac{u}{2}(1-t)$. It follows that if the half-progress relation is satisfied for all $j< i$, then $\ve_{i+1}\geq \ve\beta^i$. 
Thus, $$\Pr\left( \text{$i$ is the first value for which }X_{i+1}>\frac{X_i+ f(X_i)}{2}\right)\leq \e\left(-\alpha  m \ve^2 \beta^{2i}\right).$$ By linear divergence, there exists $L=\Omega(\log(\frac{1}{\ve}))$ such that if the sequence satisfies the half-progress relation for all $i<L$, then $X_L<u$. We bound the probability that this does not happen. Let $m_\ell= \frac{8\ln(\frac{1}{u(1-t)\gamma})}{d^2u(1-t)^2 \beta^i \ve^2}$. For ease of notation, let $c= \ln{ \frac{1}{u(1-t)\gamma}}<1$. Observe
\begin{align*} \Pr(X_L >u)&\leq \sum_{i=0}^L \e\left(-\alpha m_\ell \ve^2 \beta^{2i}\right)\\
&= \sum_{i=0}^L \e\left(-c \beta^{i}\right)\\
&\leq \sum_{i=0}^L \e\left(-c(1+iu(1-t))\right)\\
&< \e\left(-c\right) \sum_{i=0}^L e^{-iu(1-t))}\\
&< \frac{\e\left(-c\right)}{1-\e\left(-u(1-t)\right)}\\
&< \frac{\e\left(-c\right)}{u(1-t)}= \gamma.
\end{align*}\end{proof}

\begin{theorem} Consider the construction of a $t$-threshold function with linear convergence given in Theorem \ref{thm:linear} in which each level $\ell$  has $m_\ell$ items and the fraction of input items firing is at least $\ve$ from the threshold $t$. Then, with probability at least $1-\gamma$,  items at level $\Omega( \log{\frac{1}{\gamma}} + \log{\frac{1}{\ve}}))$ will accurately compute the threshold function for  $ m=\Omega\left( \ln(\frac{1}{\gamma}) ( \frac{1}{\gamma}+ \frac{1}{\ve^2})\right).$ \label{linear finite} \end{theorem}

\begin{proof} Let $X_i$ be the fraction of items firing at level $i$. Then  $E(X_i)= f(X_{i-1})$.  By Corollary \ref{cor:complementary}, it suffices to consider the case when the fraction of inputs firing is less that $t- \ve$. As proved in Theorem \ref{thm:linear}, in expectation, the sequence $X_1, X_2, X_3, \dots $ convergences to $0$. We will show that with probability at least $1- \frac{\gamma}{2}$, the sequence drops below $\frac{\gamma}{2}$. First we apply Lemma \ref{finite linear}. 
Recall that the polynomial corresponding to this construction is $f(p)= p+p(1-p)(p-t)$ and therefore $d$ in the statement of Lemma \ref{finite linear} is $1$. Let $u$ be a constant $0<u<t$, $m \geq \frac{8\ln(\frac{4}{u(1-t)\gamma})}{u(1-t)^2  \ve^2}$ and $L=\Omega(\frac{1}{\ve})$. Thus, $X_L<u$ with probability at least $1-\frac{\gamma}{4}$.

Next we show that given $X_L< u$ the probability that the sequence continues to obey the half-progress relation (as defined in Lemma \ref{finite linear}) and drops below $\frac{\gamma}{2}$ is at least $1-\frac{\gamma}{4}$.  Let $\alpha=\frac{(1-u)^2(t-u)^2}{8}$. Given a fixed value $X_i< u$, $$\Pr\left(X_{i+1}>\frac{X_i+ f(X_i)}{2}\right)<\e\left(-\alpha m X_i^2\right).$$ 
We compute the probability that $N+i$ is the first value for which the half-progress relation is not satisfied given $X_L< u$. If $X_{i}<u$ and the half-progress relation is satisfied at $i$ then $X_{i+1} \leq X_i( 1-\beta)$ where  $\beta=\frac{1}{2}(1-u)(t-u)$.  It follows that if the half-progress relation is satisfied for all $j< i$, then $X_{N+i}\leq ( 1- \beta)^{i}u$. Let $L'= \frac{4}{(1-u)(t-u)} \log_2\left(\frac{2u}{\gamma}\right)$. If for all $L\leq i \leq L+L'$, the half-progress relation is satisfied then $X_{L+{L'}}< u(1-\beta)^{L'}<\frac{\gamma}{2}$. We bound the probability that this does not happen. 
Let $m\geq \frac{16\ln{\left(\frac{16}{(1-u)(t-u)\gamma}\right)}}{(1-u)^2(t-u)^2 \gamma}$.
 For ease of notation, let $c= \ln{\left(\frac{8}{\beta\gamma}\right)}$. Observe \begin{align*} \Pr\left(X_{L+L'}>\frac{\gamma}{2}\right)&\leq \sum_{i=0}^{L'} \e\left(-mX_i \alpha\right)\\
&= \sum_{i=1}^{L'} \e\left(-\frac{2cX_i}{\gamma}\right)\\
&\leq \sum_{i=0}^{L'} \e\left(-c(1-\beta)^{-(L'-i)}\right)\\
&= \sum_{i=0}^{L'} \e\left(-c(1-\beta)^{i}\right)\\
&\leq \sum_{i=0}^{\beta L'}\frac{1}{\beta} \e\left(-ce^{i}\right)\\
&\leq \frac{2\e\left(-c\right)}{\beta}\\ 
&=\frac{\gamma}{4}.\end{align*}
Therefore, with probability at least $1-\frac{\gamma}{2}$, items at level $\Omega( \log{\frac{1}{\gamma}} + \log{\frac{1}{\ve}}))$ of an iterative construction with width $m$ fire with probability at most $\frac{\gamma}{2}$ for $m=\Omega\left( \ln(\frac{1}{\gamma}) ( \frac{1}{\gamma}+ \frac{1}{\ve^2})\right).$ Thus, the iterative construction accurately computes the threshold function with probability at least $(1-\frac{\gamma}{2})^2> 1-\gamma$. \end{proof}

We give a tighter bound for the accuracy of the finite width construction for functions with quadratic convergence. We now prove Theorem \ref{finite width quad}, which can be restated as follows: in order to accurately compute, with probability at least $1-\gamma$, a $t$-threshold function for inputs in which the fraction of inputs firing is within $\ve$ of $t$ , the width of the levels must be $\Omega\left(\frac{ \ln(1/\gamma) }{\ve^2}\right).$

\begin{proof}[of Theorem \ref{finite width quad}.]
Let $X_i$ be the fraction of items firing at level $i$. Then  $E(X_i)= f(X_{i-1})$.  By Corollary \ref{cor:complementary}, it suffices to consider the case when the fraction of inputs firing is less that $t- \ve$. As proved in Lemma \ref{conditions}, in expectation, the sequence $X_1, X_2, X_3, \dots $ convergences to $0$. We will show that with probability at least $1-\gamma$, the sequence reaches $0$. 

First we apply Lemma \ref{finite linear}.  Recall from the proof of Theorem \ref{thm:quadratic four}, that the minimum value of $g(p)=\frac{f(p)-p}{ p(1-p)(p-t)}$ is $1$ on the interval $[0,1]$.  Therefore for such constructions $d$ in the statement of Lemma \ref{finite linear} is $1$. For constructions described in Theorem \ref{thm:k leaves}, the minimum value of $g(p)=\frac{f(p)-p}{ p(1-p)(p-t)}$ is $\frac{1}{t}$ on the interval $[0,1]$, as proved in Lemma \ref{k,k+1}. Therefore for such constructions $d$ in the statement of Lemma \ref{finite linear} is $1/t$.
Let $u$ be the constant $0<u<t$ in quadratic convergence as in Lemma \ref{conditions}, $m_\ell \geq \frac{8\ln(\frac{4}{u(1-t)\gamma})}{d^2u(1-t)^2 (1+(c_1/2)^{\ell-1}\ve^2}$ and $L=\Omega(\frac{1}{\ve})$. Thus, $X_L<u$ with probability at least $1-\frac{\gamma}{2}$.

 Next, we bound the probability given $X_L<u$, that $X_{L+r}=0$. We say that $X_{k+1}$ \textit{regresses} if $X_{k+1} \geq u$. For $m \geq \frac{l}{u}$ and $c_3$ as in Lemma \ref{conditions}, we apply Lemma \ref{binomial}  and obtain \begin{align*}
\Pr(X_{k+1} \geq X_k)&= \sum_{i=\lceil\frac{u}{e}n\rceil}^m {m \choose i} f(p)^i (1-f(p))^{m-i}\\
&\leq \e\left(-m\left(p \log\left(\frac{p}{f(p)}\right)+(1-p)\log\left(\frac{1-p}{1-f(p)}\right)\right)\right)\\
&\leq \e\left(-m\left(p \log\left(\frac{1}{c_3p}\right)+(1-p)\log\left(1-\frac{p-f(p)}{1-f(p)}\right)\right)\right)\\
&\leq \e\left(-m\left(p \log\left(\frac{1}{c_3p}\right)+(1-p)\left(-\frac{p-f(p)}{1-f(p)}\right)\right)\right)\\
&\leq \e\left(-m\left(p \log\left(\frac{1}{c_3p}\right)-p(1-p)\right)\right)\\
&\leq \e\left(-mp \log\left(\frac{1}{c_3pe}\right)\right) \\
&\leq(c_3ue)^\ell\end{align*}
It follows that $$\Pr(\text{$X_L, X_{L+1}, \dots X_{L+r}$ do not regress})\geq 1-r(c_3ue)^\ell.$$ 

Next we bound the probability that given $X_k<u$, $X_{k+1}=0$. 
$$\Pr(X_{k+1}=0 | X_{k}\leq u)=(1-f(X_k))^m \geq (1-u^2)^m \geq 1-mu^2 =1-lu.$$
Therefore $$\Pr(X_{L+r}=0 | \text{$X_L, X_{L+1}, \dots X_{L+r}$ do not regress}) \geq 1- (lu)^r.$$
Let $l=r =\max\left\{ c_3 e,\min\left\{\frac{1}{2u}, \log_2\left(\frac{4}{\gamma}\right)\right\}\right\}$. It follows $\log_2\left( \frac{4}{\gamma}\right)< l \log_2\left(\frac{1}{lu}\right)$ and therefore $\frac{4}{\gamma} < \left(\frac{1}{lu}\right)^\ell$. 
We now compute \begin{align*} \Pr(X_{L+r}=0 )\geq 1-r(c_3ue)^\ell-(lu)^r 
\geq 1-2(lu)^\ell
\geq 1-\frac{\gamma}{2}.
\end{align*}

 We have shown that given $X_L<u$, $\Pr(X_{L+r}>0)\leq \frac{\gamma}{2}$. Therefore, with probability at least $1-\gamma$,  items at level $\Omega( \log{\frac{1}{\gamma}} + \log{\frac{1}{\ve}}))$ do not fire for $m=\Omega\left(\frac{ \ln(\frac{1}{\gamma}) }{\ve^2}\right).$ \end{proof}
 
\subsection{Exponential and wild constructions}

In this section, we analyze exponential constructions, where items are chosen with probabilities proportional to their weights, and the latter decay exponentially with time, by a factor of $e^{-\alpha}$. 
%
The wild construction is the special case of the exponential construction in which the $\alpha=0$, i.e. the weights of the items do not decay with time. Here we analyze the wild and exponential constructions for the probability distribution given in Theorem \ref{thm:linear}, which converges linearly to a $t-$threshold function. Analogous results, up to constants, hold for the constructions with quadratic convergence. 

Throughout this section, we use the following notation. We define $X_k$ to be the probability an item chosen according to the weights at the beginning of the $k+1$ iteration is 1. We define $$d_i= \sum_{j=0}^{i-1} e^{-j \alpha}$$ $$W_k=n e^{-\alpha k} +\sum_{j=0}^{k-1} e^{-j \alpha}.$$ 
 Additionally, we use the following concentration inequality. 
\begin{lemma} (Hoeffding) Let $X_1,\dots, X_n$ be independent random variables with $ X_i \in [a_i, b_i]$, and $X=\sum_{i=1}^n X_i$. Then for $t>0$, $$\Pr[ X > \E[X]+t] \leq \e\left(- \frac{2t^2}{\sum_{i=1}^n (a_i-b_i)^2 }\right).$$ \label{Hoeffding}
\end{lemma}

The following two lemmas lay the foundation for the proofs of Theorem \ref{wild} and \ref{exponential}, which prove convergence for the wild and exponential constructions respectively. 

\begin{lemma} Consider an exponential construction corresponding to a $t$-threshold function given in Theorem \ref{thm:linear}. Let 
 $X_k$ be the probability an item chosen according to the weights at the beginning of the $k+1$ iteration is 1, and let $\ve=t- X_k$. Let $A_i$ be the event that $X_{k+i} \leq X_k + \frac{\ve}{2}$. Then  $$\Pr[ \bigwedge_{i=1}^s A_i]\geq \prod_{i=1}^s1- \e\left( \frac{-\ve^2 \beta_i^2}{2d_i} \right)$$ where $\beta_i= n e^{-(k+i)\alpha} + \sum_{j=i}^{k+i-1} e^{-j \alpha} = W_{k+i}-d_i$.  \label{x obeys} \end{lemma}

\begin{proof} 
Note $$\Pr[ \bigwedge_{i=1}^s A_i]= \Pr[A_s| \bigwedge_{j=1}^{s-1} A_i]\Pr[\bigwedge_{j=1}^{s-1}A_j]=\prod_{i=1}^s \Pr[A_i |  \bigwedge_{j=1}^{i-1} A_j].$$
We now bound $\Pr[A_i |  \bigwedge_{j=1}^{i-1} A_j].$
Let $y_1, \dots y_i$ be random variables where $y_j$ takes value $\frac{e^{-(i-j)\alpha}}{d}$ if the $(k+j)^{th}$ item fires. Let $Y= \sum_{j=1}^i y_j$. 
We compute $$X_{k+i}= \frac{W_k X_k e^{-i\alpha} + d Y}{W_{k+i}}.$$ The statement $X_{k+i}\geq X_k+ \frac{\ve}{2}$ is equivalent to the following statements, with $d=d_i$: 
$$Y \geq \frac{W_{k+i}-W_k e^{-i \alpha}}{d} X_k +\frac{\ve}{2}\frac{W_{k+i}}{d}$$ $$Y \geq X_k + \frac{\ve}{2} + \frac{\ve}{2}\left( \frac{ n e^{-(k+i)\alpha} + \sum_{j=i}^{k+i-1} e^{-j \alpha}}{d}\right)=X_k + \frac{\ve}{2} + \frac{\ve\beta_i}{2d}$$ 
Note that $\E[Y]\le X_k+\frac{\ve}{2}.$
We have \begin{align*}
\Pr[\overline{A_i} |  \bigwedge_{j=1}^{i-1} A_j]=\Pr\left[ Y \geq X_k + \frac{\ve}{2} + \frac{\ve\beta_i}{2d} \right]&\leq \Pr\left[ Y- \E[Y] \geq   \frac{\ve\beta_i}{2d} \right]\end{align*}
Note that $0\leq y_i \leq \frac{e^{-(i-j)\alpha}}{d}$.
Applying Hoeffding, we obtain
 \begin{align*}
\Pr\left[Y- \E[Y] \geq  \frac{\ve \beta_i}{2d}\right]
&\leq \e\left(-2\left(\frac{\ve\beta_i}{2d}\right)^2\frac{1}{d(1/d^2)}\right)\\
&= \e\left( \frac{-\ve^2 \beta_i^2}{2d} \right).
\end{align*}
It follows that $$\Pr[ \bigwedge_{i=1}^s A_i]\geq \prod_{i=1}^s1- \e\left( \frac{-\ve^2 \beta_i^2}{2d_i} \right).$$
\end{proof}

\begin{lemma} Consider an exponential construction corresponding to a $t$-threshold function given in Theorem \ref{thm:linear}. Let 
 $X_k$ be the probability an item chosen according to the weights at the beginning of the $k+1$ iteration is 1, and let $\ve=t-X_k$. Assume $X_{k+i} 
\leq X_k +\frac{ \ve}{2}$ for all $1\leq i \leq s$. Then $$ \E[ X_{k+s}]\leq X_k \prod_{i=1}^s\left( 1-\frac{\ve(1-t)}{2 W_{k+i}}\right).$$\label{expected}\end{lemma}

\begin{proof} For $X_{k+i-1} \leq X_k +\frac{ \ve}{2}$, we use the definition of $f$ to compute $$f(X_{k+i-1})= X_{k+i-1}-X_{k+i-1}(1-X_{k+i-1})(t-X_{k+i-1})\leq X_{k+i-1}\left( 1-\frac{\ve(1-t)}{2}\right).$$ We compute $$\E[ X_{k+i}]= \frac{W_{k+i-1}X_{k+i-1}e^{-\alpha} +f( X_{k+i-1})}{W_{k+i}} \leq X_{k+i-1}\left( 1-\frac{\ve(1-t)}{2W_{k+i}}\right).$$
\end{proof}

We are now ready to prove Theorem \ref{wild} and Theorem \ref{exponential} which describe the number of items needed to guarantee high accuracy in the wild and exponential constructions respectively. 

\begin{proof}[of Theorem \ref{wild}.] Assume that initial fraction of inputs firing is below the target threshold. The other case follows similarly. We divide the analysis into phases in each of which the total number of items double; in phase $j$, $2^jn$ new items are created.
 Let $X_k$ be the fraction of the first $k$ items and the $n$ inputs that fire. 
First we show that with high probability, $X_{k+i} < X_k +\frac{\ve}{2}$ holds for the entire phase, meaning for all $1\leq i \leq n+k$. We call this event $A_j$ where $j$ denotes the phase. Lemma \ref{x obeys} for  $\alpha=0$ states $$\Pr[A]\geq \left(1- e^{\frac{-\ve^2 (n+k)}{2}}\right)^{(n+k)} \geq 1- (n+k)e^{\frac{-\ve^2 (n+k)}{2}}.$$
We bound the probability that $A_j$ does not occur for some phase $$ \Pr\left[ \bigvee_{j} \overline{A}_j\right]\leq \sum_{j} \Pr[ \overline{A}_j]\leq \sum_j (2^j n)\e\left(- \frac{\ve^2 2^jn}{2}\right)<2n\e\left( \frac{-\ve^2n}{2}\right) < \frac{\delta}{4}.$$ 

Next, we compute the expected progress of the sequence of $X_k's$ corresponding to the final items in each phase. Setting $\alpha=0$, Lemma \ref{expected} states $$\E[ X_{k+s}] \leq X_k \left( 1- \frac{ \ve (1-t)}{2(n+k+s)} \right)^s \leq X_k \left( 1- \frac{ \ve (1-t)s}{4(n+k+s)}\right) =  X_k \left( 1- \frac{ \ve (1-t)}{8}\right).$$

 We now show that with high probability, the actual progress made by each phase is close to its expectation.  Let $z_1, \dots z_k$ be indicator random variables where $z_i=1$ if $i$ the $(k+i)^{th}$ item fires, and let $Z= \sum_{i=1}^s z_i$. We apply the Hoeffding bound and obtain \begin{align*} \Pr\left[ X_{k+s} - \E[X_{k+s}]> \frac{\ve X_k(1-t)}{16}\right] &= \Pr\left[ Z - \E[Z]> \frac{2s\ve X_k(1-t)}{16}\right] \\
&<\e\left(\frac{-\ve^2 (1-t)^2 X_k^2 s }{128}\right).  \end{align*}
Let $B_j$ be the event that  $X_{2^{j+1}n} - \E[X_{2^{j+1}n}]\leq \frac{\ve X_k(1-t)}{16}$. 
Thus, 
$$ \Pr\left[ \bigvee_{j} \overline{B}_j\right]\leq \sum_{j} \Pr[ \overline{B}_j]\leq \sum_j \e\left(\frac{-\ve^2 (1-t)^2 X_k^2 2^j n }{128}\right) < \frac{\delta}{4}.$$ 

We have shown, with probability at least $1-\frac{\delta}{2}$, the events $A_j$ and $B_j$ hold until the fraction of items firing is less that $\frac{\delta}{2}$. At this point, an item does not fire and therefore correctly computes the threshold function with probability $1-\frac{\delta}{2}$. 

To bound the total number of items, we compute the number of phases  needed to achieve $X_k< \frac{\delta}{2}.$ We divide the analysis into two parts, when $\epsilon = t-X_k < X_k$ and $X_k < t-X_k$. In the first part, $\eps = t-X_k$ grows by a constant factor in each phase, while in the second part $X_k$ decays by a constant factor in each phase. So the total number of phases needed is $O\left(\log(1/\delta)+\log(1/ \ve)\right)$. The total number of items therefore is $O\left(n(\frac{1}{\delta^c}+\frac{1}{\ve^c})\right)$ for some absolute constant $c$.\end{proof}

\begin{proof}[of Theorem \ref{exponential}.]
As before, we assume without loss of generality that the initial fraction of inputs firing is below the threshold. We divide the analysis into phases where $s$ is the number of items created in a phase, $k$ is the number of items created prior to the start of the phase, and $n$ is the number of inputs. Let $X_k$ be the probability an item selected according to the probability distribution in the $(k+1)^{st}$ iteration fires. We will set the length of a phase $s = \lfloor1/\alpha\rfloor$. 

First we compute the probability that, $X_{k+i} < X_k +\frac{\ve}{2}$ holds for the entire phase, meaning for all $1\leq i \leq s$.  We call this event $A_j$ where $j$ denotes the phase. Recall the definition of $\beta_{i}$ given in Lemma \ref{x obeys}. For $1 \leq i \leq s=1/\alpha$ and $n > 1/\alpha$, we have $$\beta_{i,k} = ne^{-(k+i)\alpha} + \sum_{j=i}^{k+i-1}e^{-j\alpha} =ne^{-(k+i)\alpha} + e^{-i\alpha}\sum_{j=i}^{k-1}e^{-j\alpha} \ge ne^{-(k+i)\alpha} + \frac{1}{e}\sum_{j=0}^{k-1}e^{-j\alpha} \ge \frac{1}{e^2\alpha}$$ where the final inequality follows from the observations that if $k < \frac{1}{\alpha}$ then $ne^{-(k+i)\alpha}\geq \frac{1}{e\alpha}$, and if $k \geq \frac{1}{\alpha}$, then $\sum_{j=0}^{k-1}e^{-j\alpha} \geq \frac{1}{e\alpha}$.
Note that $d_i < d_s=\frac{1-e^{-1}}{1-e^{-\alpha}} <\frac{1}{\alpha}$, so $\beta_{i,k} >\frac{d_i}{e}$. 
Lemma \ref{x obeys} implies that $$\Pr[A_j]\geq 1- \prod_{i=1}^s1- \e\left( \frac{-\ve^2 \beta_{i,k}^2}{2d_i} \right) \geq \left(1- \sum_{i=1}^s\e\left(- \frac{\ve^2\beta_{i,k}^2}{2d_i}\right) \right) \ge  \left(1-    \sum_{i=1}^s\e\left(- \frac{\ve^2\beta_{i,k}}{2e}\right) \right) .$$
Therefore $$\Pr\left[ \overline{A}_j\right] \leq  \sum_{i=1}^s\e\left(- \frac{\ve^2\beta_{i,k}}{2e}\right) \leq  \frac{1}{\alpha} \e\left(- \frac{\ve^2}{2e^3\alpha}\right). $$


Next, we compute the probability that the progress made in a phase is close to its expectation. We ignore the progress made in the first $\log n$ phases. Let $B_j$ be the event that $X_{k+s}-\E\left[ X_{k+s}\right] \leq \frac{\ve_j X_k (1-t)}{32}$ where $ k =\frac{j-1}{\alpha}$, $s=\frac{1}{\alpha}$, and $\ve_j= t- X_k$,  given $\overline{A}_j$ and $j \geq \log n$.
First note that  
$$W_{k+i} = ne^{-(k+i)\alpha}+\sum_{j=0}^{k+i-1}e^{-j\alpha} \le ne^{-k\alpha}+\frac{2}{\alpha}.$$ 
Lemma \ref{expected} implies that 
$$\E[ X_{k+s}] \leq  X_k \prod_{i=s/2}^s\left( 1-\frac{\ve_j(1-t)}{2 W_{k+i}}\right) \leq X_k \left(1 - \frac{\ve_j(1-t)}{2(ne^{-k\alpha}+\frac{2}{\alpha})}\right)^s \leq X_k \left( 1- \frac{ \ve_j (1-t)s}{4(ne^{-k\alpha}+\frac{2}{\alpha})}\right).$$
For $j\geq \log n$ phases, we have $k\ge \log n/\alpha$ and therefore,
\[
\E[X_{k+s}] \le X_k \left(1-\frac{\eps_j(1-t)}{16}\right).
\]
We now show that with high probability, the actual progress made by each phase is close to its expectation.  Let $z_1, \dots z_s$ be indicator random variables where $z_i=\frac{e^{-(s-i) \alpha}}{d_s}$ if $i$ the $(k+i)^{th}$ item fires, and let $Z= \sum_{i=1}^s z_i$.  Note $$X_{k+s} - \E[X_{k+s}] = \frac{d_s}{W_{k+s}} \left(Z- \E[Z] \right).$$
We apply the Hoeffding bound, noting that $\sum_{i=1}^s  \left(\frac{e^{-(s-i) \alpha}}{d_s}\right)^2 \le \frac{1}{\alpha d_s^2}$, and obtain 
\begin{align*}\Pr\left[ \overline{B}_j\right]= \Pr\left[ X_{k+s} - \E[X_{k+s}]> \frac{\ve X_k(1-t)}{32}\right] &= \Pr\left[ Z - \E[Z]> \frac{W_{k+s}\ve X_k(1-t)}{32d_s}\right] \\
&<\e\left(\frac{-\ve^2 (1-t)^2 X_k^2 W_{k+s}^2\alpha}{512}\right)\\
&< \e\left(\frac{-\ve^2(1-t)^2X_k^2}{512e^4 \alpha}\right)\\
&= \e\left(\frac{-\ve^2(1-t)^2X_k^2 2048 e^4 \log(\frac{4}{\ve\delta})}{512e^4 \min\{ \ve^2, \delta^2\} }\right)\\
  &<\frac{(\delta \ve)^4}{256} \end{align*}
The third step uses the fact that $W_{k+s} \ge \frac{1}{\alpha \ve^2}$.

Next we compute the total number of phases needed before $X_k < \delta/3$ given $A_j$ holds for all phases and assuming $B_j$ holds after the first $\log n$ phases. For each phase after the first $\log n$ phases, we have $$X_{k+s} \leq \left( 1- \frac{ \ve_j(1-t)}{32} \right) X_k.$$ We use the observation that if $y\geq \frac{1}{x}$, then $(1-x)^y \leq \frac{1}{e}$. Observe that $$t \left(1- \frac{\ve(1-t)}{32}\right)^{\lceil \frac{32}{\ve(1-t)} \rceil} \leq \frac{t}{e} < \frac{t}{2}.$$ Thus, after $\lceil \frac{32}{\ve(1-t)} \rceil$ phases, the current $X_k$ is less than $t/2$. Similarly we compute $$\frac{t}{2} \left(1-\frac{t/2(1-t)}{32}\right)^{\lceil \frac{64\log{\frac{3t}{2\delta}}}{t(1-t)}\rceil}\leq \frac{t}{2}\left( \frac{1}{e}\right)^{\log \frac{3t}{2 \delta}}=\frac{\delta}{3}.$$ and conclude that after an additional $\lceil \frac{64\log{3t/2\delta}}{t(1-t)}\rceil$ phases the current $X_k$ is less than $\delta/3$. We may assume $\ve+\delta< t< 1-\ve-\delta$. The following is an upper bound on the total number of phases needed to have a failure probability less than $\delta/3$: $$T= \log n+  \frac{32}{\ve^2} + \frac{96}{\delta^2}.$$

It suffices to show that $(i)$ the probability $A_j$ holds for the first $T$ phases is at least $1-\delta/3$ and $(ii)$ the probability $B_j$ holds for the latter $T- \log n$ phases is at least $1-\delta/3$. 
We show $(i)$ in two parts. First we analyze the probability of some event $\overline{A}_j$ in the for the first $\log n$ phases. We have $$\Pr\left[ \bigvee_{j=1}^{\log n} \overline{A}_j \right] \leq   \sum_{j=1}^{\log n} \frac{1}{\alpha} \e\left(\frac{-\ve^2 \max_i \{\beta_{i,\frac{j-1}{\alpha}}\} }{2e}\right).$$
For ease of notation, let $a= \e\left( -\frac{\ve^2}{2e}\right)$. For $j=1$, $\beta_{i,0} \geq \frac{n}{e}$. For $j>1$, note $\beta_{i,\frac{j-1}{\alpha}} \geq \frac{n}{e^j}+\frac{1}{e^2\alpha}$. We rewrite the above expression. \begin{align*}\Pr\left[ \bigvee_{j=1}^{\log n} \overline{A}_j \right]\leq  a^{\frac{n}{e}} +a^{\frac{1}{e^2\alpha}}\sum_{j=2}^{\log n} a^{\frac{n}{e^j}}=a^{\frac{n}{e}} +a^{\frac{1}{e^2\alpha}} \left( \frac{1-a}{1-a^{1/e}} \right)\leq 3 a^{\frac{1}{e^2\alpha}}.
\end{align*}
Observe $$ \e\left( -\frac{\ve^2}{2e^3 \alpha}\right)= \e\left( - \frac{\ve^2 2048e^4\log 4/\ve \delta}{2e^3\min\{ \ve^2, \delta^2\}}\right)\leq \frac{(\ve \delta)^{1024 e}}{256}.$$
Next we compute \begin{align*}
\Pr\left[ \bigvee_{j=1}^{T} \overline{A}_j \right]&\leq \left( 3  + \frac{T-\log n}{\alpha} \right)  \e\left( -\frac{\ve^2}{2e^3 \alpha}\right)\\
  &\leq \left( 3  + \left(  \frac{32}{\ve^2} + \frac{96}{\delta^2}\right) \left(  \frac{ 2048 e^4 \log\frac{4}{\ve \delta}}{\min\{\ve^2, \delta^2\}}\right)  \right) \frac{(\ve \delta)^{1024 e}}{256}\\
& < \frac{\delta}{3}.\end{align*}
Statement $(ii)$ follows from the calculation:$$\Pr\left[ \bigvee_{j=\log n}^{T} \overline{B_j}\right]\leq \left(  \frac{32}{\ve^2} + \frac{96}{\delta^2}\right)\frac{(\ve\delta )^4}{256}< \frac{\delta}{3}. $$

We have shown after $T= \Omega(\log (n/\eps\delta)$ phases the next item will fire with probability less than $\delta$. Since each phase has $1/\alpha$ new items, the total new items created to achieve this is $$\Omega\left(\frac{\log \frac{1}{\eps\delta}}{\min \{\eps^2, \delta^2\}}\left(\log \frac{n}{\eps\delta}\right)\right).$$
 \end{proof}

Finally, we note the lemmas and theorems in this section hold up to modification of constants for wild and exponential constructions corresponding to constructions given in Theorem \ref{thm:quadratic four} and Theorem \ref{thm:k leaves}, which exhibit  quadratic convergence in the iterative tree setting. However, unlike in the leveled iterative construction, in the exponential and wild constructions we do not see asymptotically faster convergence than the corresponding linear constructions, even in the quadratic regime. In our analysis of the latter constructions, we track the probability that the sequence of $X_k's$ goes to zero faster than iterating on curve $g(p)$ that is a weighted average of $f(p)$ and the function $p$. Regardless of whether $f(p)$ has quadratic behavior, $g(p)$ has a linear term, implying $g(p)$ exhibits linear convergence. Since this yardstick function $g(p)$ exhibits linear convergence, this analysis will not yield asymptotically improved results for functions with quadratic convergence.

\section{Learning}\label{learning}

So far we have studied the realizability of thresholds via neurally plausible simple iterative constructions. These constructions were based on prior knowledge of the target threshold. Here we study the learnability of thresholds from examples. It is important that the learning algorithm should be neurally plausible and not overly specialized to the learning task. We believe the simple results presented here are suggestive of considerably richer possibilities.

We begin with a one-shot learning algorithm.  We show that given a single example of a string $X \in \{0,1\}^n$ with $\|X\|_1=tn$, we can build an iterative tree that computes a $t$-threshold function with high probability. Let $T_1$ and $T_2$ be the building block trees in the construction given in Theorem \ref{thm:linear}. The simple LearnThreshold algorithm, described below, has the guarantee stated in Theorem \ref{thm:learning}, which follows from Theorem \ref{thm:linear}.

\begin{figure}[h!]
\fbox{\parbox{\textwidth}{
{\bf LearnThreshold($L,m,X$):}\\
\textit{Input:} Levels parameter $L$, a string $X \in \{0,1\}^n$ such that $\|X\|_1=tn$, width parameter $m$.\\
\textit{Output:} A finite realization of iterative tree with width $m$.\\

For each level $j$ from $1$ to $L$, apply the following iteration $m$ times:\\ 
(level $0$ consists of the input items $X$)
\begin{enumerate}
\item Pick a random input item $i$. 
\item If $X_i=1$ then let $T=T_1$, else let $T=T_2$. 
\item Pick $3$ items from the previous level. 
\item Build $T$ with these items as leaves. 
\end{enumerate}
}}
\end{figure}

\section{Discussion}\label{sec:discuss}

We have seen that very simple, distributed algorithms requiring minimal global coordination and control can lead to stable and efficient constructions of important classes of functions.  Our work raises several interesting questions. 

\begin{enumerate}

\item What are the ways in which threshold functions are applied in cognition? Object recognition is one application of threshold functions in cognition. For instance, suppose we have items representing features such as ``trunk," ``grey," ``wrinkled skin," and ``big ears," and an item representing our concept of an ``elephant." If a certain threshold of items representing the features we associate with an elephant fire, then the ``elephant" item will fire. This structure lends itself to a hierarchical organization of concepts that is consistent with the fact that as we learn, we build on our existing set of knowledge. For example, when a toddler learns to identify an elephant, he does not need to re-learn how to identify an ear. The item representing ``ear" already exists and will fire as a result of some threshold function created when the toddler learned to identify ears.  Now the item representing ``ear" may be used as an input as the toddler learns to identify elephants and other animals.

\item What is an interesting model and neurally plausible algorithm for learning threshold functions of $k$ relevant input items? 
In this scenario, the input is a set of sparse binary strings of length $n$ representing examples in which at least $tk$ of $k$ relevant items are firing. The output is an iterative tree that computes a $t$-threshold function on the $k$ relevant items. 
We can formulate the previously described example of  learning to identify an elephant as an instance of this problem. Each time the toddler sees an example of an elephant, many features associated with elephant will fire in addition to some features that are not associated with elephants. There may also be features associated with an elephant that are not present in the example and therefore not firing. 
A learning algorithm must rely on information about the items that are currently firing to learn both the set of relevant items and a threshold function on this set of items. It might also be beneficial to utilize prediction, as e.g., done by \cite{PV15}.

\item To what extent can general linear threshold functions with general weights be constructed/learned by cortical algorithms? 

\item A concrete question is whether the construction of Theorem \ref{thm:k leaves} is optimal, similar to the optimality of the constructions in Theorem \ref{thm:quadratic four}.


\item A simple way to include non monotone Boolean functions with the same constructions as we study here, would be to have input items together with their negations (as in e.g., \citep{Savicky1990}). What functions can be realized this way, using a distribution on a small set of fixed-size trees?

\end{enumerate}
\bibliographystyle{plainnat}
\bibliography{neuro}

\newpage

\section{Appendix} 
\appendix

\begin{table}[h!]
 \begin{center}
\begin{tabular}{|l|l|} \hline
Degree & Polynomials in $\mathcal{A}$ \\ \hline
1 & (0,1)\\ \hline
2 & (0, 0, 1)\\
& (0, 2, -1)\\ \hline
3 &(0, 0, 0, 1) \\
& (0, 1, 1, -1)\\
 & (0, 0, 2, -1) \\
 & (0, 3, -3, 1))\\ \hline
 4 & ((0, 0, 0, 0, 1)\\
 &    (0, 1, 0, 1, -1)\\
   &  (0, 0, 1, 1, -1)\\
    & (0, 2, 0, -2, 1)\\
    & (0, 0, 0, 2, -1)\\
   &  (0, 1, 2, -3, 1)\\
    & (0, 0, 3, -3, 1)\\
    & (0, 4, -6, 4, -1)\\
    & (0, 0, 2, 0, -1)\\
    & (0, 0, 4, -4, 1))\\ \hline
\end{tabular}
\quad
\begin{tabular}{|l|l|} \hline
Degree & Polynomials in $\mathcal{A}$ \\ \hline
 5 & (0, 0, 0, 0, 0, 1) \\
  &   (0, 1, 0, 0, 1, -1)\\
   &  (0, 0, 1, 0, 1, -1)\\
&     (0, 2, -1, 1, -2, 1)\\
  &   (0, 0, 0, 1, 1, -1)\\
    & (0, 1, 1, 0, -2, 1)\\
&     (0, 0, 2, 0, -2, 1)\\
  &   (0, 3, -2, -2, 3, -1)\\
    & (0, 0, 0, 0, 2, -1)\\
&     (0, 1, 0, 2, -3, 1)\\
  &   (0, 0, 1, 2, -3, 1)\\
    & (0, 2, 1, -5, 4, -1)\\
&     (0, 0, 0, 3, -3, 1)\\
  &   (0, 1, 3, -6, 4, -1)\\
   &  (0, 0, 4, -6, 4, -1)\\
 &    (0, 5, -10, 10, -5, 1)\\
  &    (0, 0, 0, 2, 0, -1)\\
 &    (0, 1, 2, -2, -1, 1)\\
 &    (0, 0, 0, 4, -4, 1)\\
 &    (0, 1, 4, -8, 5, -1)\\
 &    (0, 0, 1, 1, 0, -1)\\
 &    (0, 0, 3, -1, -2, 1)\\
 &    (0, 0, 2, 1, -3, 1)\\
 &    (0, 0, 6, -9, 5, -1)\\ \hline
\end{tabular}
\end{center}\caption{ Achievable polynomials for AND/OR trees.  The polynomial $a_0 + a_1 x + a_2 x^2 +a_3 x^3 + a_4 x^4 + a_5 x^5$ is denoted by $(a_0, a_1, a_2, a_3, a_4, a_5)$.}\label{the table}
\end{table}

\end{document}